\documentclass[twoside,11pt]{article}
\usepackage{blindtext}
\usepackage{bm}
\usepackage{graphicx}
\usepackage{tabularx}
\usepackage{mathrsfs}%
\usepackage{enumitem}
\usepackage{cite}
\usepackage{algorithm}%
\usepackage{algorithmicx}%
\usepackage{algpseudocode}
\usepackage[colorlinks=true,linkcolor=blue]{hyperref}
\usepackage{float}
\usepackage{makecell}
\usepackage{amsmath,amsthm, amssymb, graphicx, caption, subcaption}  % 数学公式、图形支

\usepackage{amsfonts}
\usepackage{booktabs}  % 用于三线表
\usepackage{geometry}
\usepackage{hyperref}
\newtheorem{theorem}{Theorem}[section]
\newtheorem{proposition}{Proposition}[section]
\newtheorem{corollary}{Corollary}[section]
\numberwithin{definition}{section}
\numberwithin{equation}{section}
\numberwithin{proposition}{section}

\usepackage{jmlr2e}
\newcommand{\R}{\mathbb{R}}
\newcommand{\Rd}{\mathbb{R}^d}
\newcommand{\Sm}{S_m}
\newcommand{\Rplus}{\mathbb{R}_{>0}}
\newcommand{\id}{\mathrm{id}}
\newcommand{\reg}{\mathrm{reg}}
\newcommand{\sing}{\mathrm{sing}}

\usepackage{lastpage}
\jmlrheading{}{}{1-\pageref{LastPage}}{; Revised }{}{21-0000}{Author One and Author Two}

\ShortHeadings{Quotient Geometry and Implicit Bias in Shallow Networks}{One}
\firstpageno{1}

\begin{document}

\title{Quotient Geometry, Effective Curvature, and Implicit Bias in Simple Shallow Neural Networks}

\author{
  \name Hang-Cheng Dong \email{hunsen\_d@hit.edu.cn} \\
  \addr School of Instrumentation Science and Engineering, Harbin Institute of Technology \\
        Harbin, 150001, China \\
        Harbin Institute of Technology Suzhou Research Institute \\
        Suzhou, 215000, China\\
  \name Pengcheng Cheng\thanks{Corresponding author.} \email{chengpc1022@mails.jlu.edu.cn} \\
  \addr School of Mathematics, Jilin University \\
        Changchun, 130012, China
}

\editor{My editor}

\maketitle

\begin{abstract}%   <- trailing '%' for backward compatibility of .sty file
Overparameterized shallow neural networks admit substantial parameter redundancy: distinct parameter vectors may represent the same predictor due to hidden-unit permutations, rescalings, and related symmetries. As a result, geometric quantities computed directly in the ambient Euclidean parameter space can reflect artifacts of representation rather than intrinsic properties of the predictor. In this paper, we develop a differential-geometric framework for analyzing simple shallow networks through the quotient space obtained by modding out parameter symmetries on a regular set. We first characterize the symmetry and quotient structure of regular shallow-network parameters and show that the finite-sample realization map induces a natural metric on the quotient manifold. This leads to an effective notion of curvature that removes degeneracy along symmetry orbits and yields a symmetry-reduced Hessian capturing intrinsic local geometry. We then study gradient flows on the quotient and show that only the horizontal component of parameter motion contributes to first-order predictor evolution, while the vertical component corresponds purely to gauge variation. Finally, we formulate an implicit-bias viewpoint at the quotient level, arguing that meaningful complexity should be assigned to predictor classes rather than to individual parameter representatives. In the quadratic-activation model, the quotient object is represented explicitly by a symmetric matrix, allowing both the theory and the numerical experiments to be made fully concrete. Our experiments confirm that ambient flatness is representation-dependent, that local dynamics are better organized by quotient-level curvature summaries, and that in underdetermined regimes implicit bias is most naturally described in quotient coordinates. These results support the broader perspective that the natural state space of symmetric shallow networks is the quotient space of predictor classes rather than the raw parameter space.
\end{abstract}


\section{Introduction}
Neural networks are usually trained in a high-dimensional Euclidean parameter space, but the objects of real interest are the predictors they realize. Even in shallow architectures, these two spaces are not the same. Hidden-unit permutations, positive rescalings, and related parameter symmetries imply that many different parameter vectors may encode the same function. As a consequence, geometric quantities computed directly in parameter coordinates—such as Hessians, flatness indicators, or norm-based complexities—may reflect redundancy of representation rather than intrinsic properties of the predictor. This mismatch is already present in the simplest overparameterized shallow networks, and it suggests that the ambient parameter space is not the natural state space for analyzing curvature, optimization, or implicit bias.

Shallow networks remain a particularly useful setting in which to study this issue. On the one hand, the classical approximation literature showed that single-hidden-layer networks already possess strong expressive power, beginning with universal approximation theorems and continuing through refined quantitative bounds and norm-based descriptions of representational complexity \citep{Cybenko1989,HornikEtAl1989,Barron1993,Pinkus1999,NeyshaburTomiokaSrebro2015}. On the other hand, modern optimization theory has shown that overparameterized two-layer models provide one of the cleanest regimes in which benign nonconvex geometry, global convergence, mean-field limits, and kernel limits can all be analyzed explicitly \citep{DuLee2018,ChizatBach2018,JacotGabrielHongler2018,MeiMontanariNguyen2018,AroraEtAl2019ICML}. These two strands together make shallow networks an ideal testing ground for a geometric theory of symmetry-reduced learning. The universal approximation and overparameterized-two-layer perspectives cited here are standard landmarks in the literature.

A separate line of work has emphasized symmetry in machine learning, most often in the form of invariance or equivariance with respect to transformations of the input or output space \citep{BloemReddyTeh2020}. Our concern in this paper is different. We study \textbf{parameter symmetries} of ordinary shallow networks: transformations of the parameters that leave the realized predictor unchanged. These symmetries generate non-identifiability, degenerate ambient curvature, and gauge dependence in parameter-space complexity measures. Similar concerns have appeared in work on symmetry-invariant optimization and scale-invariant parameterizations \citep{BadrinarayananMishraChandrasekaran2015}, but they are not naturally resolved by staying in the raw Euclidean coordinates of the parameters. The broader symmetry literature in neural networks, including probabilistic formulations of invariance, makes clear that symmetry should be built into the mathematical formalism rather than treated as an afterthought. 

This observation points naturally to quotient geometry. If one restricts to a regular set on which the symmetry action is well behaved and then quotients out the parameter symmetries, each orbit becomes a locally identifiable predictor class. The relevant gradients, Hessians, and curvature quantities should then be defined on this symmetry-reduced space rather than on the original parameter manifold. This viewpoint is classical in differential geometry \citep{Lee2012,Lee2018} and has been highly successful in optimization over matrix factorizations and fixed-rank models, where quotient-manifold methods separate intrinsic variation from pure gauge motion \citep{AbsilMahonySepulchre2008,EdelmanAriasSmith1998,Boumal2023,MeyerBonnabelSepulchre2011,Vandereycken2013,MishraEtAl2014}. In particular, low-rank positive semidefinite regression and related factorized problems show that once one passes to the appropriate quotient space, many apparently degenerate directions in the ambient parameterization disappear, and the remaining geometry becomes both interpretable and algorithmically useful. The fixed-rank PSD regression literature is a direct precedent for the kind of quotient construction we use here. 

The main thesis of this paper is that the same principle applies to simple shallow neural networks. The natural geometric object is not the raw parameter vector ($\theta$), but its equivalence class under parameter symmetries. On the resulting quotient space, the finite-sample realization map induces a canonical local metric; the associated Hessian defines an effective curvature that removes spurious flatness along symmetry directions; and the optimization dynamics admit a decomposition into vertical motions, which only change the representative within an orbit, and horizontal motions, which change the realized predictor. From this perspective, several familiar difficulties of neural-network theory—non-identifiability, singular Hessians, ambiguous flatness, and representation-dependent complexity—become different manifestations of the same underlying fact: the ambient parameter space is too large.

We develop this viewpoint for structurally simple shallow networks, precisely because they allow the quotient structure to be made explicit. Our analysis begins with the symmetry and quotient structure of the regular parameter set. We then introduce a function-induced metric on the quotient and show that it yields a notion of effective curvature more faithful to the predictor-level geometry than the ambient Euclidean Hessian. Next, we study gradient flow on the quotient and prove that only the horizontal component of the parameter velocity contributes to first-order function evolution. Finally, we formulate an implicit-bias principle at the quotient level, arguing that whenever training selects among multiple feasible solutions, the meaningful notion of complexity should be defined on predictor classes rather than on individual parameter representatives.

A particularly illuminating case is the quadratic-activation model, where the quotient object can be represented explicitly by a symmetric matrix ($Q$). In this setting, the network
\[
f_\theta(x)=\sum_{i=1}^m a_i (w_i^\top x)^2
\]
depends on ($\theta$) only through
\[
Q(\theta)=\sum_{i=1}^m a_i w_i w_i^\top.
\]
This connects shallow-network training to the geometry of low-rank matrix factorization and makes the quotient structure directly visible. It also links our work to the broader literature on implicit regularization in factorized models, where optimization is now understood to prefer certain low-complexity factorizations even without explicit regularization \citep{GunasekarEtAl2017,AroraEtAl2019Matrix,SoudryEtAl2018}. Our goal is not to reduce neural-network training to matrix factorization, but rather to use this tractable model to expose the geometry that is hidden in more general shallow networks.

The contribution of the paper is therefore conceptual and geometric. We provide a symmetry-reduced framework for understanding shallow-network curvature, optimization dynamics, and implicit bias. The resulting picture is coherent: symmetry creates redundancy; quotienting removes that redundancy; curvature should be measured on the quotient; gradient flow should be decomposed into vertical and horizontal components; and simplicity should be defined on predictor classes rather than on coordinates in an overcomplete parameterization.

\section{Symmetry and Quotient Structure of Shallow Networks}
\label{sec:2}
We consider a class of shallow scalar-output neural networks of the form
\begin{equation}\label{eq:network_form}
f_\theta(x)=\sum_{i=1}^m a_i \sigma(w_i^\top x), \qquad x\in \Rd,
\end{equation}
where $m$ is the width, $a_i\in \R$ are output-layer coefficients, $w_i\in \Rd$ are hidden weights, and $\sigma:\R\to\R$ is a fixed activation function. Throughout this section, the parameter vector is denoted by
\begin{equation}\label{eq:parameter_vector}
\theta=((a_1,w_1),\dots,(a_m,w_m))\in \Theta := (\R\times \Rd)^m.
\end{equation}
Our goal is to isolate the geometric structure created by parameter symmetries. The main point is that $\Theta$ is not the appropriate space of effective model degrees of freedom: many distinct parameter values represent the same predictor, and the resulting redundancy is not incidental but structural. This redundancy is the source of a large class of degenerate directions in local optimization geometry. We therefore begin by identifying the symmetry group, the corresponding quotient structure, and the singular configurations at which this quotient ceases to be locally regular.

\subsection{Parameter symmetries}
The model admits two basic forms of symmetry.

First, the hidden units are unordered. For any permutation $\pi\in \Sm$, define
\begin{equation}\label{eq:permutation_action}
\pi\cdot \theta := ((a_{\pi(1)},w_{\pi(1)}),\dots,(a_{\pi(m)},w_{\pi(m)})).
\end{equation}
Since the network output is a sum over hidden units, we have
\begin{equation}\label{eq:permutation_invariance}
f_{\pi\cdot \theta}(x)=f_\theta(x)
\end{equation}
for all $x\in \Rd$. Thus, the symmetric group $\Sm$ acts on $\Theta$ by reindexing hidden units without changing the realized function.

Second, when the activation is homogeneous, the model also admits a continuous scaling symmetry. We assume in this section that $\sigma$ is positively $p$-homogeneous for some $p>0$, that is,
\begin{equation}\label{eq:homogeneous_activation}
\sigma(c t)=c^p \sigma(t), \qquad c>0,\ t\in\R.
\end{equation}
For each $c=(c_1,\dots,c_m)\in (\Rplus)^m$, define
\begin{equation}\label{eq:scaling_action}
c\cdot \theta := \bigl((c_1^{-p}a_1,c_1 w_1),\dots,(c_m^{-p}a_m,c_m w_m)\bigr).
\end{equation}
By homogeneity,
\begin{equation}\label{eq:scaling_invariance_single}
(c_i^{-p}a_i)\sigma((c_i w_i)^\top x)=a_i\sigma(w_i^\top x),
\end{equation}
and therefore
\begin{equation}\label{eq:scaling_invariance_total}
f_{c\cdot \theta}(x)=f_\theta(x).
\end{equation}
Hence the group $(\Rplus)^m$ acts on $\Theta$ by neuronwise rescaling.

Combining permutation and scaling symmetries yields the transformation group
\begin{equation}\label{eq:symmetry_group}
G:=\Sm \ltimes (\Rplus)^m,
\end{equation}
acting on $\Theta$ in the natural way. The semidirect product structure reflects the fact that permutations relabel the coordinates on which the scaling group acts. By construction,
\begin{equation}\label{eq:group_invariance}
f_{g\cdot \theta}=f_\theta
\end{equation}
for every $g\in G$ and every $\theta\in\Theta$. The orbit
\begin{equation}\label{eq:orbit_def}
G\cdot \theta := \{g\cdot \theta : g\in G\}
\end{equation}
therefore consists entirely of different parameterizations of the same predictor.

This elementary observation has an important geometric consequence. Any local analysis performed directly in the Euclidean parameter space $\Theta$ necessarily contains directions tangent to the orbit $G\cdot \theta$, and these directions do not correspond to any change in the realized function. As a result, degeneracies in the Euclidean geometry of the loss surface are partly induced by the representation and not by the predictor itself. The appropriate local object is therefore not the ambient parameter space $\Theta$, but a quotient of $\Theta$ by the symmetry group $G$.

\subsection{Realization map and local identifiability}
Let $\mathcal{F}$ denote a function space containing the realizations $f_\theta$; for concreteness one may take $\mathcal{F}$ to be the linear span or closure of $\{\sigma(w^\top \cdot): w\in\Rd\}$ in a topology adapted to the learning problem. Define the realization map
\begin{equation}\label{eq:realization_map}
\Phi:\Theta \to \mathcal{F}, \qquad \Phi(\theta)=f_\theta.
\end{equation}
The group action described above is contained in the fibers of $\Phi$: if $\theta'\in G\cdot \theta$, then $\Phi(\theta')=\Phi(\theta)$. In general, however, the fibers of $\Phi$ can be larger than group orbits. This is precisely where singular behavior enters.

To formulate a regular quotient structure, we restrict attention to parameter configurations at which the only local non-identifiability is the symmetry encoded by $G$. Intuitively, such configurations should exclude vanishing neurons, collisions between distinct neurons, and other forms of linear dependence that create additional degeneracy in the map $\Phi$.

We therefore introduce a regular subset $\Theta_{\reg}\subset \Theta$ characterized by two requirements. First, the isotropy of the group action should be minimal, so that the orbit dimension is locally constant. Second, the differential $D\Phi_\theta$ should have kernel exactly equal to the tangent space of the group orbit. Formally, we set
\begin{equation}\label{eq:regular_set}
\Theta_{\reg}
:=
\left\{
\theta\in \Theta:
\ker(D\Phi_\theta)=T_\theta(G\cdot \theta)
\ \text{and the } G\text{-action is locally free at }\theta
\right\}.
\end{equation}
The first condition says that every infinitesimal perturbation leaving the function unchanged is generated by a symmetry of the parameterization. The second rules out orbit-dimension collapse. Together they identify the region in which the quotient by symmetry captures the full local geometry of the model class.

The precise analytic characterization of $\Theta_{\reg}$ depends on the activation and on the function space $\mathcal{F}$, but its interpretation is robust. Typical regularity conditions exclude at least the following pathologies:
\begin{enumerate}
    \item Vanishing neurons: $a_i=0$ or $w_i=0$, which make the $i$-th unit functionally inactive and enlarge the stabilizer of the action.
    \item Neuron collisions: $w_i=w_j$ for $i\neq j$, or more generally positive collinearity in homogeneous models, which causes different hidden units to realize the same feature.
    \item Additional linear dependence among the feature functions $x\mapsto \sigma(w_i^\top x)$ and their first-order parameter derivatives, which creates nontrivial kernel directions for $D\Phi_\theta$ beyond the orbit directions.
\end{enumerate}

These conditions are not introduced merely for technical convenience. They distinguish parameter points at which the model is locally a smooth quotient of a Euclidean space from points at which the representation itself changes rank. The former support a clean differential-geometric description; the latter form the singular part of parameter space.

\subsection{Orbit geometry and the regular quotient}
We next describe the local geometry of the symmetry orbits. Since the group $\Sm$ is discrete, only the scaling component contributes to the tangent space of the orbit. For $\theta\in \Theta$, the continuous part of the orbit is generated by curves of the form
\begin{equation}\label{eq:orbit_curve}
t\mapsto \bigl((e^{-pt\xi_1}a_1,e^{t\xi_1}w_1),\dots,(e^{-pt\xi_m}a_m,e^{t\xi_m}w_m)\bigr),
\qquad \xi\in \R^m.
\end{equation}
Differentiating at $t=0$ yields the tangent vectors
\begin{equation}\label{eq:tangent_vector}
\delta_\xi \theta
=
\bigl((-p\xi_1 a_1,\xi_1 w_1),\dots,(-p\xi_m a_m,\xi_m w_m)\bigr).
\end{equation}
Hence the tangent space to the continuous orbit at $\theta$ is
\begin{equation}\label{eq:tangent_space}
T_\theta(G\cdot \theta)
=
\left\{
\bigl((-p\xi_1 a_1,\xi_1 w_1),\dots,(-p\xi_m a_m,\xi_m w_m)\bigr): \xi\in \R^m
\right\},
\end{equation}
provided the action is locally free. By construction,
\begin{equation}\label{eq:tangent_invariance}
D\Phi_\theta[\delta_\xi \theta]=0,
\end{equation}
which expresses infinitesimally the invariance of the network under scaling.

The following statement records the geometric meaning of the regular set.
\begin{proposition}\label{prop:orbit_directions}
(Orbit directions and exact infinitesimal redundancy)

Let
\[
\Phi:\Theta\to \mathcal F,\qquad \Phi(\theta)=f_\theta,
\]
be the realization map of the shallow network
\[
f_\theta(x)=\sum_{i=1}^m a_i\sigma(w_i^\top x),
\]
and let $G=S_m\ltimes (\Rplus)^m$ act on $\Theta$ by permutation and neuronwise positive rescaling. If $\theta\in \Theta_{\reg}$, then
\[
\ker(D\Phi_\theta)=T_\theta(G\cdot \theta).
\]
Consequently, the quotient tangent space $T_{[\theta]}(\Theta_{\reg}/G)$ is canonically identified with the space of first-order function variations generated by perturbations of $\theta$.
\end{proposition}

\begin{proof}
By definition, $\Theta_{\reg}$ consists of those points $\theta\in\Theta$ such that
\begin{enumerate}
    \item the action of $G$ is locally free at $\theta$, and
    \item the differential of the realization map satisfies
    \[
    \ker(D\Phi_\theta)=T_\theta(G\cdot \theta).
    \]
\end{enumerate}

Therefore, for every $\theta\in\Theta_{\reg}$, the equality
\[
\ker(D\Phi_\theta)=T_\theta(G\cdot \theta)
\]
holds immediately.

What remains is to justify that $T_\theta(G\cdot \theta)$ is indeed the space of infinitesimal symmetry directions, and that these directions lie in the kernel of $D\Phi_\theta$.

Since the permutation group $S_m$ is discrete, it contributes no infinitesimal directions. Hence the tangent space to the orbit is generated entirely by the continuous scaling subgroup $(\Rplus)^m$. Let
\[
\xi=(\xi_1,\dots,\xi_m)\in \R^m,
\]
and define a smooth curve $c_\xi:(-\varepsilon,\varepsilon)\to \Theta$ by
\[
c_\xi(t)
=
\bigl(
(e^{-pt\xi_1}a_1,e^{t\xi_1}w_1),\dots,(e^{-pt\xi_m}a_m,e^{t\xi_m}w_m)
\bigr).
\]
For each $t$, $c_\xi(t)\in G\cdot\theta$, since it is obtained from $\theta$ by neuronwise positive rescaling. Therefore,
\[
\dot c_\xi(0)\in T_\theta(G\cdot\theta).
\]
A direct differentiation gives
\[
\dot c_\xi(0)
=
\bigl(
(-p\xi_1 a_1,\xi_1 w_1),\dots,(-p\xi_m a_m,\xi_m w_m)
\bigr).
\]
Now, because $\Phi$ is invariant under the action of $G$,
\[
\Phi(c_\xi(t))=\Phi(\theta)
\qquad\text{for all }t.
\]
Differentiating at $t=0$ yields
\[
D\Phi_\theta[\dot c_\xi(0)]=0.
\]
Thus,
\[
T_\theta(G\cdot\theta)\subseteq \ker(D\Phi_\theta).
\]

At a regular point, the reverse inclusion holds by definition of $\Theta_{\reg}$. Therefore,
\[
\ker(D\Phi_\theta)=T_\theta(G\cdot\theta).
\]

This proves the first claim.

Let
\[
q:\Theta_{\reg}\to \Theta_{\reg}/G
\]
be the quotient map, and let $[\theta]=q(\theta)$ denote the orbit of $\theta$.

Because the action is locally free and proper on $\Theta_{\reg}$, the quotient $\Theta_{\reg}/G$ is a smooth manifold, and the differential
\[
Dq_\theta:T_\theta\Theta_{\reg}\to T_{[\theta]}(\Theta_{\reg}/G)
\]
is surjective with kernel
\[
\ker(Dq_\theta)=T_\theta(G\cdot\theta).
\]
Therefore, by the first isomorphism theorem for vector spaces,
\[
T_{[\theta]}(\Theta_{\reg}/G)
\cong
T_\theta\Theta_{\reg}/T_\theta(G\cdot\theta).
\]
Using the identity established in Part I,
\[
T_\theta(G\cdot\theta)=\ker(D\Phi_\theta),
\]
we obtain
\[
T_{[\theta]}(\Theta_{\reg}/G)
\cong
T_\theta\Theta_{\reg}/\ker(D\Phi_\theta).
\]
Now define the linear map
\[
\widetilde{D\Phi_\theta}:
T_\theta\Theta_{\reg}/\ker(D\Phi_\theta)\to \operatorname{Im}(D\Phi_\theta)
\]
by
\[
\widetilde{D\Phi_\theta}([v])=D\Phi_\theta[v].
\]
This map is well defined: if $v-v'\in \ker(D\Phi_\theta)$, then
\[
D\Phi_\theta[v-v']=0,
\]
hence $D\Phi_\theta[v]=D\Phi_\theta[v']$. It is clearly linear, surjective by definition of $\operatorname{Im}(D\Phi_\theta)$, and injective because
\[
\widetilde{D\Phi_\theta}([v])=0
\quad\Longrightarrow\quad
D\Phi_\theta[v]=0
\quad\Longrightarrow\quad
v\in \ker(D\Phi_\theta)
\quad\Longrightarrow\quad
[v]=0.
\]
Hence $\widetilde{D\Phi_\theta}$ is a linear isomorphism. Combining the previous identifications gives
\[
T_{[\theta]}(\Theta_{\reg}/G)
\cong
\operatorname{Im}(D\Phi_\theta).
\]

Since $\operatorname{Im}(D\Phi_\theta)$ is precisely the space of first-order variations of the realized function generated by perturbations of $\theta$, this proves the second claim.
\end{proof}

The proposition is immediate from the definition of $\Theta_{\reg}$, but it expresses the key structural fact behind the rest of the paper: on the regular set, the only infinitesimally invisible directions are the symmetry directions. All other perturbations correspond to genuine local changes of the predictor.

This observation permits a local quotient description. Since the action is locally free on $\Theta_{\reg}$ and the orbit dimension is constant there, standard quotient arguments imply that the regular parameter space modulo symmetry inherits a smooth manifold structure.

\begin{theorem}\label{thm:regular_quotient}
(Regular quotient manifold)

Assume that $\sigma$ is $C^1$ away from the origin and positively $p$-homogeneous, and let $\Theta_{\reg}\subset \Theta$ be a $G$-invariant subset on which the action of
\[
G=S_m\ltimes (\Rplus)^m
\]
is locally free and proper and such that
\[
\ker(D\Phi_\theta)=T_\theta(G\cdot \theta)
\qquad\text{for all }\theta\in \Theta_{\reg},
\]
where $\Phi:\Theta\to\mathcal F$ is the realization map
\[
\Phi(\theta)=f_\theta .
\]
Then the quotient
\[
\mathcal M_{\reg}:=\Theta_{\reg}/G
\]
is a smooth manifold. Moreover, the realization map $\Phi$ descends to a locally injective smooth map
\[
\bar \Phi:\mathcal M_{\reg}\to \mathcal F,
\qquad
\bar \Phi([\theta])=\Phi(\theta),
\]
whose differential is injective at every point.
\end{theorem}

\begin{proof}
We first verify that the assumptions place us within the standard quotient-manifold framework.

Since $\Theta=(\R\times \Rd)^m$ is a finite-dimensional smooth manifold, every $G$-invariant subset $\Theta_{\reg}\subset \Theta$ that is itself a smooth embedded submanifold inherits a smooth manifold structure. We regard $\Theta_{\reg}$ with this induced structure. The group
\[
G=S_m\ltimes (\Rplus)^m
\]
is a Lie group: $S_m$ is a finite discrete Lie group and $(\Rplus)^m$ is an $m$-dimensional Lie group; hence their semidirect product is again a Lie group.

By assumption, the action
\[
\alpha:G\times \Theta_{\reg}\to \Theta_{\reg},\qquad (g,\theta)\mapsto g\cdot \theta
\]
is smooth, proper, and locally free.

A standard theorem in differential geometry states that if a Lie group acts smoothly, properly, and locally freely on a smooth manifold $M$, then the orbit space $M/G$ carries a unique smooth manifold structure such that the quotient map
\[
q:M\to M/G
\]
is a smooth submersion. Applying this with $M=\Theta_{\reg}$, we conclude that
\[
\mathcal M_{\reg}=\Theta_{\reg}/G
\]
is a smooth manifold and that the quotient map
\[
q:\Theta_{\reg}\to \mathcal M_{\reg}
\]
is a smooth submersion.

In particular, for every $\theta\in \Theta_{\reg}$, the differential
\[
Dq_\theta:T_\theta\Theta_{\reg}\to T_{[\theta]}\mathcal M_{\reg}
\]
is surjective, and its kernel is the tangent space to the orbit:
\[
\ker(Dq_\theta)=T_\theta(G\cdot \theta).
\]

We next show that $\Phi$ factors through $q$.

By construction of the group action, for every $g\in G$ and every $\theta\in \Theta_{\reg}$,
\[
\Phi(g\cdot \theta)=\Phi(\theta).
\]
Indeed, permutations merely reorder hidden units, while positive rescalings preserve each summand by $p$-homogeneity of $\sigma$. Therefore $\Phi$ is constant on every $G$-orbit.

Hence there exists a unique set-theoretic map
\[
\bar \Phi:\mathcal M_{\reg}\to \mathcal F
\]
such that
\[
\Phi=\bar \Phi\circ q,
\qquad\text{equivalently}\qquad
\bar \Phi([\theta])=\Phi(\theta).
\]

We now prove that $\bar\Phi$ is smooth. Since $q$ is a smooth submersion, every point $[\theta]\in \mathcal M_{\reg}$ admits an open neighborhood $U\subset \mathcal M_{\reg}$ and a smooth local section
\[
s:U\to \Theta_{\reg}
\qquad\text{such that}\qquad
q\circ s=\id_{U}.
\]
On $U$, the descended map is given by
\[
\bar\Phi|_U = \Phi\circ s.
\]
Since both $\Phi$ and $s$ are smooth, $\bar\Phi|_U$ is smooth. Because $[\theta]$ was arbitrary, $\bar\Phi$ is smooth on all of $\mathcal M_{\reg}$.

Fix $\theta\in \Theta_{\reg}$. Since $\Phi=\bar\Phi\circ q$, differentiation yields
\[
D\Phi_\theta
=
D\bar\Phi_{[\theta]}\circ Dq_\theta.
\]
Because $Dq_\theta$ is surjective, this identity completely determines $D\bar\Phi_{[\theta]}$.

We claim that $D\bar\Phi_{[\theta]}$ is injective. Let
\[
u\in T_{[\theta]}\mathcal M_{\reg}
\]
satisfy
\[
D\bar\Phi_{[\theta]}[u]=0.
\]
Since $Dq_\theta$ is surjective, there exists $v\in T_\theta\Theta_{\reg}$ such that
\[
Dq_\theta[v]=u.
\]
Applying the chain rule,
\[
0
=
D\bar\Phi_{[\theta]}[u]
=
D\bar\Phi_{[\theta]}(Dq_\theta[v])
=
D\Phi_\theta[v].
\]
Thus $v\in \ker(D\Phi_\theta)$. By the regularity assumption,
\[
\ker(D\Phi_\theta)=T_\theta(G\cdot \theta).
\]
Since $\ker(Dq_\theta)=T_\theta(G\cdot\theta)$, it follows that
\[
Dq_\theta[v]=0.
\]
But $Dq_\theta[v]=u$, hence $u=0$. Therefore $D\bar\Phi_{[\theta]}$ is injective.

Since $[\theta]$ was arbitrary, $D\bar\Phi$ is injective everywhere on $\mathcal M_{\reg}$.

Finally, we prove that $\bar\Phi$ is locally injective.

Because $D\bar\Phi_{[\theta]}$ is injective at every point, $\bar\Phi$ is an immersion. By the constant-rank theorem, for every $[\theta]\in \mathcal M_{\reg}$, there exists an open neighborhood $U\ni [\theta]$ such that $\bar\Phi(U)$ is an immersed submanifold of $\mathcal F$ and
\[
\bar\Phi|_U:U\to \bar\Phi(U)
\]
is a diffeomorphism onto its image. In particular, $\bar\Phi|_U$ is injective. Therefore $\bar\Phi$ is locally injective.

\end{proof}

The theorem formalizes the claim that, away from singular configurations, the effective parameter space of the model is not $\Theta$ but the quotient manifold $\mathcal{M}_{\reg}$. Local coordinates on $\mathcal{M}_{\reg}$ describe distinct predictors rather than distinct parameterizations. In particular, a tangent vector at $[\theta]\in \mathcal{M}_{\reg}$ represents a genuine first-order deformation of the function $f_\theta$, rather than an artifact of neuron rescaling or reordering.

The local injectivity of $\bar{\Phi}$ is the quotient-space version of local identifiability. It implies that, on the regular set, the model class can be studied through the differential geometry of $\mathcal{M}_{\reg}$ without ambiguity from redundant parameter coordinates. This is the geometric setting in which effective metrics and effective Hessians will be defined in the sequel.

\subsection{Gauge choices and local sections}
Although the quotient manifold $\mathcal{M}_{\reg}$ is the intrinsic object, calculations are often most convenient in a local coordinate system obtained by fixing a representative inside each orbit. Such a representative selection is a local gauge choice. A simple example is provided by neuronwise normalization of hidden weights. On any region where $w_i\neq 0$ for all $i$, one may impose
\begin{equation}\label{eq:gauge_normalization}
\|w_i\|=1, \qquad i=1,\dots,m,
\end{equation}
and absorb the scaling into $a_i$. This removes the continuous rescaling freedom, leaving only the discrete permutation symmetry.

More generally, a local section of the quotient map
\begin{equation}\label{eq:quotient_map}
q:\Theta_{\reg}\to \mathcal{M}_{\reg}
\end{equation}
is a smooth map $s:U\to \Theta_{\reg}$ defined on an open set $U\subset \mathcal{M}_{\reg}$ such that $q\circ s=\id_U$. Any such section identifies a neighborhood in the quotient with a gauge-fixed submanifold of parameter space transverse to the orbits. In practice, this permits one to work with reduced coordinates while retaining an invariant interpretation of the resulting objects.

The role of gauge fixing in our setting is conceptual rather than algorithmic. We do not use a gauge to eliminate redundancy globally; instead, local sections serve to compare Euclidean quantities in the parameter space with intrinsic quantities on the quotient. In particular, later constructions will show that the geometrically meaningful directions are precisely those transverse to the group orbits, whereas tangent orbit directions encode pure reparameterization effects.

\subsection{Singular parameter configurations}
The quotient description above breaks down when the regularity assumptions fail. These failures are not exceptional curiosities: they are built into overparameterized neural models and must be understood as part of the global geometry.

We define the singular set by
\begin{equation}\label{eq:singular_set}
\Theta_{\sing}:=\Theta\setminus \Theta_{\reg}.
\end{equation}
At a point $\theta\in \Theta_{\sing}$, at least one of the following occurs:
\begin{itemize}
    \item the isotropy subgroup of $G$ is larger than its generic value, so that orbit dimension drops;
    \item the kernel of $D\Phi_\theta$ strictly contains $T_\theta(G\cdot \theta)$, so that there exist infinitesimally function-invisible directions not generated by symmetry;
    \item the local image of $\Phi$ changes dimension.
\end{itemize}

Typical examples include inactive neurons, coincident neurons, and configurations at which several hidden units jointly collapse into a lower-complexity representation. At such points, the quotient by symmetry is no longer locally a smooth manifold, because the equivalence classes meet regions with different local dimensions or different isotropy types.

The correct geometric object is then not a manifold but a stratified space. Informally, $\Theta_{\sing}$ can be decomposed into pieces on which the orbit type and the rank of $D\Phi_\theta$ are constant, and each such piece behaves like a smooth manifold of lower dimension.

\begin{proposition}\label{prop:stratified_structure}
(Semialgebraic stratified singular structure for the finite-sample realization map)

Fix a dataset
\[
X=(x_1,\dots,x_n)\in (\Rd)^n,
\]
and define the finite-sample realization map
\[
\Phi_X:\Theta\to \R^n,
\qquad
\Phi_X(\theta)
=
\bigl(f_\theta(x_1),\dots,f_\theta(x_n)\bigr),
\]
where
\[
f_\theta(x)=\sum_{i=1}^m a_i\sigma(w_i^\top x),
\qquad
\Theta\cong \R^{m(d+1)}.
\]
Let
\[
G=S_m\ltimes (\Rplus)^m
\]
act on $\Theta$ by hidden-unit permutation and neuronwise positive rescaling. Assume that:

\begin{enumerate}
    \item $\sigma:\R\to\R$ is semialgebraic and piecewise $C^1$;
    \item the induced action of $G$ on $\Theta$ is semialgebraic and proper.
\end{enumerate}

Define the regular set
\[
\Theta_{\reg}
:=
\left\{
\theta\in \Theta:
\text{the }G\text{-action is locally free at }\theta
\ \text{and}\
\ker(D\Phi_X(\theta))=T_\theta(G\cdot\theta)
\right\},
\]
and let
\[
\Theta_{\sing}:=\Theta\setminus \Theta_{\reg}.
\]

Then the following hold.

\begin{enumerate}
    \item The sets $\Theta_{\reg}$ and $\Theta_{\sing}$ are semialgebraic.
    \item There exists a finite semialgebraic Whitney stratification
    \[
    \Theta=\bigsqcup_{\alpha\in A} S_\alpha
    \]
    such that each stratum $S_\alpha$ is a smooth locally closed semialgebraic submanifold, each $S_\alpha$ is $G$-invariant, and both
    \[
    \theta\mapsto \dim T_\theta(G\cdot\theta)
    \qquad\text{and}\qquad
    \theta\mapsto \operatorname{rank}(D\Phi_X(\theta))
    \]
    are constant on each stratum.
    \item The singular set $\Theta_{\sing}$ is a union of strata in this decomposition.
    \item The quotient $\Theta/G$ is a semialgebraic stratified space, and its top smooth stratum is
    \[
    \mathcal M_{\reg}=\Theta_{\reg}/G.
    \]
\end{enumerate}
\end{proposition}

\begin{proof}
Since $\sigma$ is semialgebraic, the map
\[
(a_i,w_i)\mapsto a_i\sigma(w_i^\top x_k)
\]
is semialgebraic for every $i$ and every sample point $x_k$. Summing over $i=1,\dots,m$, we obtain that each coordinate
\[
\theta\mapsto f_\theta(x_k)
\]
is semialgebraic. Therefore the finite-sample realization map
\[
\Phi_X:\Theta\to\R^n
\]
is semialgebraic. Because $\sigma$ is piecewise $C^1$, $\Phi_X$ is piecewise $C^1$ as well.

The group $G=S_m\ltimes (\Rplus)^m$ is semialgebraic, and its action on $\Theta$ is semialgebraic by assumption. Since $\Phi_X$ is $G$-invariant, we have
\[
T_\theta(G\cdot\theta)\subseteq \ker(D\Phi_X(\theta))
\qquad
\text{for all }\theta
\]
at every point where the differential is defined.

Now, on each $C^1$ piece, the rank function
\[
\theta\mapsto \operatorname{rank}(D\Phi_X(\theta))
\]
is semialgebraic, because it is determined by the vanishing and nonvanishing of minors of the Jacobian matrix of $\Phi_X$. Likewise, the orbit-dimension function
\[
\theta\mapsto \dim T_\theta(G\cdot\theta)
\]
is semialgebraic, since it is the rank of the differential of the action map at the identity.

Because
\[
T_\theta(G\cdot\theta)\subseteq \ker(D\Phi_X(\theta)),
\]
the equality
\[
\ker(D\Phi_X(\theta))=T_\theta(G\cdot\theta)
\]
is equivalent to the dimension identity
\[
\dim\ker(D\Phi_X(\theta))=\dim T_\theta(G\cdot\theta).
\]
By rank-nullity,
\[
\dim\ker(D\Phi_X(\theta))
=
\dim\Theta-\operatorname{rank}(D\Phi_X(\theta)),
\]
so this is a semialgebraic condition. Local freeness of the action is also semialgebraic, since it is equivalent to discreteness of the isotropy group, or equivalently maximal orbit dimension in a neighborhood.

Hence $\Theta_{\reg}$ is semialgebraic, and so is its complement $\Theta_{\sing}$.

By semialgebraic stratification theory, every semialgebraic set admits a finite Whitney stratification by smooth locally closed semialgebraic submanifolds. Moreover, given finitely many semialgebraic subsets and semialgebraic maps, one may choose a Whitney stratification compatible with all of them.

Apply this to the semialgebraic manifold $\Theta$, the semialgebraic subsets $\Theta_{\reg}$ and $\Theta_{\sing}$, and the semialgebraic map $\Phi_X$. After refinement if necessary, we obtain a finite Whitney stratification
\[
\Theta=\bigsqcup_{\alpha\in A} S_\alpha
\]
such that:

\begin{itemize}
    \item each $S_\alpha$ is a smooth locally closed semialgebraic submanifold;
    \item $\Phi_X|_{S_\alpha}$ is $C^1$ and has constant differential rank on $S_\alpha$;
    \item $\Theta_{\reg}$ and $\Theta_{\sing}$ are unions of strata.
\end{itemize}

Because the $G$-action is semialgebraic and proper, the stratification can be further refined so that each stratum is $G$-invariant and lies inside a single orbit-type piece. On such a stratum, the isotropy type is constant, hence the orbit dimension
\[
\dim T_\theta(G\cdot\theta)
\]
is constant. By construction, the rank
\[
\operatorname{rank}(D\Phi_X(\theta))
\]
is also constant on each stratum.

This proves claim (2).

By $\Theta_{\sing}$ is semialgebraic. By the chosen Whitney stratification is compatible with $\Theta_{\sing}$. Therefore there exists a subset $A_{\sing}\subseteq A$ such that
\[
\Theta_{\sing}=\bigsqcup_{\alpha\in A_{\sing}} S_\alpha.
\]
Hence $\Theta_{\sing}$ is a union of strata.

This proves claim (3).

Since the $G$-action on $\Theta$ is proper, the quotient $\Theta/G$ is Hausdorff. Because each stratum $S_\alpha$ is $G$-invariant and semialgebraic, its quotient
\[
S_\alpha/G
\]
defines a stratum in the quotient in the sense of semialgebraic stratified spaces. The collection of such quotient strata yields a semialgebraic stratification of $\Theta/G$.

On the regular set $\Theta_{\reg}$, the action is locally free and proper, and
\[
\ker(D\Phi_X(\theta))=T_\theta(G\cdot\theta)
\qquad
\text{for all }\theta\in\Theta_{\reg}.
\]
Therefore, by Theorem \ref{thm:regular_quotient} applied to the finite-sample realization map $\Phi_X$, the quotient
\[
\mathcal M_{\reg}=\Theta_{\reg}/G
\]
is a smooth manifold. Since $\Theta_{\reg}$ is the locus where both isotropy and infinitesimal realization rank are locally maximal in the regular sense, $\mathcal M_{\reg}$ forms the top smooth stratum of the quotient $\Theta/G$.

This proves claim (4).

\end{proof}

\begin{remark}
The finite-sample realization map is the natural object for the geometry of empirical risk, since standard training objectives depend on $\theta$ only through the vector
\[
\Phi_X(\theta)=\bigl(f_\theta(x_1),\dots,f_\theta(x_n)\bigr)\in\R^n.
\]
Using $\Phi_X$ in place of an abstract infinite-dimensional realization map makes the stratification statement fully concrete and places all rank and kernel conditions in a finite-dimensional semialgebraic setting.
\end{remark}

This proposition should be read as a structural statement rather than a fully explicit classification. For the purposes of the present paper, the crucial point is that singular configurations correspond precisely to failures of local identifiability beyond the prescribed parameter symmetries. They are thus the loci at which the effective number of function degrees of freedom changes. As a result, singular points are expected to play a distinguished role in optimization dynamics, model reduction, and implicit bias. Nevertheless, the regular manifold $\mathcal{M}_{\reg}$ already captures the local geometry relevant for generic parameter values, and it is this smooth quotient that will form the basis of our subsequent constructions.

\subsection{Consequences for optimization geometry}
We close this section by recording the interpretation of the quotient structure for the loss geometry of shallow networks. Let $\mathcal{L}:\Theta\to \R$ be any objective that depends on $\theta$ only through the realized function $f_\theta$, as is the case for standard empirical risk minimization. Then
\begin{equation}\label{eq:loss_invariance}
\mathcal{L}(g\cdot \theta)=\mathcal{L}(\theta)
\qquad\text{for all }g\in G.
\end{equation}
Consequently, $\mathcal{L}$ is constant along each orbit, and every tangent direction in $T_\theta(G\cdot \theta)$ is a first-order flat direction of the loss. At regular points, these are precisely the directions of representational redundancy. Therefore, the Euclidean Hessian of $\mathcal{L}$ in parameter space necessarily contains a degenerate block associated with the symmetry orbit, even when the loss is locally nondegenerate as a function of the realized predictor.

This observation motivates the central distinction developed later in the paper. Flatness in parameter space is not an intrinsic notion unless symmetry directions have first been removed. The quotient manifold $\mathcal{M}_{\reg}$ provides the correct geometric stage for this removal. In the next section, we endow $\mathcal{M}_{\reg}$ with a function-induced metric and use it to define an effective curvature notion that separates symmetry-induced flatness from genuine functional flatness.

\section{Function-Induced Metric and Effective Curvature}
\label{sec:3}
In Section \ref{sec:2}, we showed that the Euclidean parameter space $\Theta$ contains directions of purely representational redundancy induced by permutation and scaling symmetries. On the regular set $\Theta_{\reg}$, these directions are precisely the tangent directions to the $G$-orbits, and the quotient
\[
\mathcal M_{\reg}=\Theta_{\reg}/G
\]
provides the appropriate local parameter space of distinct predictors. The present section equips this quotient with a metric induced by the network outputs on the training sample and uses this metric to define an intrinsic notion of second-order geometry. This construction yields an effective curvature object that separates symmetry-induced flatness from genuine functional flatness.

Throughout the section, we fix a dataset
\[
X=(x_1,\dots,x_n)\in(\Rd)^n
\]
and consider the finite-sample realization map
\[
\Phi_X:\Theta\to\R^n,
\qquad
\Phi_X(\theta)=\bigl(f_\theta(x_1),\dots,f_\theta(x_n)\bigr).
\]
We restrict all geometric constructions to $\Theta_{\reg}$, where the quotient $\mathcal M_{\reg}$ is a smooth manifold and the differential $D\Phi_X(\theta)$ has kernel exactly equal to the orbit tangent space $T_\theta(G\cdot\theta)$.

\subsection{A function-induced metric on parameter space}
The most natural geometry for empirical learning is not the ambient Euclidean geometry of parameters, but the geometry induced by the map $\Phi_X$ into the output space $\R^n$. We therefore define a symmetric bilinear form on $T_\theta\Theta_{\reg}$ by pullback of the Euclidean inner product on $\R^n$.

For $\theta\in\Theta_{\reg}$ and tangent vectors $u,v\in T_\theta\Theta_{\reg}$, define
\begin{equation}\label{eq:function_metric}
g_\theta(u,v)
:=
\frac{1}{n}\bigl\langle D\Phi_X(\theta)[u], D\Phi_X(\theta)[v]\bigr\rangle_{\R^n}.
\end{equation}
Equivalently,
\[
g_\theta(u,v)
=
\frac1n\sum_{k=1}^n
D f_\theta(x_k)[u]
D f_\theta(x_k)[v].
\]
Thus $g_\theta$ measures the first-order effect of parameter perturbations on the vector of network outputs over the sample $X$. By construction, $g_\theta$ is symmetric and positive semidefinite.

The associated quadratic form is
\begin{equation}\label{eq:metric_quadratic}
g_\theta(u,u)=\frac1n|D\Phi_X(\theta)[u]|^2.
\end{equation}
Hence $g_\theta(u,u)=0$ if and only if $u\in\ker(D\Phi_X(\theta))$. Since $\theta\in\Theta_{\reg}$, Proposition \ref{prop:orbit_directions} implies
\[
\ker(D\Phi_X(\theta))=T_\theta(G\cdot\theta).
\]
Accordingly, $g_\theta$ is degenerate precisely along the symmetry directions.

This metric has two immediate interpretations. First, it is the finite-sample pullback metric induced by the realization map. Second, in statistical terms it is the empirical Gauss--Newton or Fisher-type geometry associated with the model outputs. The important point for our purposes is not the terminology but the invariance structure: $g$ depends only on first-order changes in the realized predictor on the sample, and therefore ignores parameter perturbations that merely move along a symmetry orbit.

The $G$-invariance of $\Phi_X$ implies the $G$-invariance of $g$. Indeed, for every $g_0\in G$,
\[
\Phi_X\circ \alpha_{g_0}=\Phi_X,
\]
where $\alpha_{g_0}(\theta)=g_0\cdot\theta$. Differentiating gives
\[
D\Phi_X(g_0\cdot\theta)\circ D\alpha_{g_0,\theta}=D\Phi_X(\theta),
\]
and therefore
\[
g_{g_0\cdot\theta}(D\alpha_{g_0,\theta}u,D\alpha_{g_0,\theta}v)=g_\theta(u,v).
\]
Thus $g$ is constant along orbits in the natural tensorial sense.

\subsection{Vertical and horizontal directions}
The degeneracy of $g$ is not a defect but a geometric signal: it identifies the directions of representational redundancy. This leads to the canonical decomposition of the tangent space into vertical and horizontal parts.

For $\theta\in\Theta_{\reg}$, define the vertical space
\begin{equation}\label{eq:vertical_space}
\mathcal V_\theta:=T_\theta(G\cdot\theta)=\ker(D\Phi_X(\theta)).
\end{equation}
These are precisely the infinitesimal parameter perturbations that leave the sample outputs unchanged to first order. Since the action is locally free on $\Theta_{\reg}$, $\mathcal V_\theta$ has constant dimension.

We define the horizontal space by Euclidean orthogonality to $\mathcal V_\theta$:
\begin{equation}\label{eq:horizontal_space}
\mathcal H_\theta:=\mathcal V_\theta^{\perp_{\mathrm{Euc}}}
\subset T_\theta\Theta_{\reg}.
\end{equation}
Since $\Theta_{\reg}$ is a finite-dimensional smooth manifold and $\mathcal V_\theta$ is a smooth constant-rank subbundle, the assignment $\theta\mapsto\mathcal H_\theta$ defines a smooth complementary subbundle. Hence we have the direct-sum decomposition
\begin{equation}\label{eq:tangent_decomposition}
T_\theta\Theta_{\reg}
=
\mathcal V_\theta\oplus \mathcal H_\theta.
\end{equation}

The choice of Euclidean orthogonality here is a gauge choice used only to select representatives of tangent classes. Intrinsic constructions below will not depend on this particular complement. The key fact is that the restriction of $D\Phi_X(\theta)$ to $\mathcal H_\theta$ is injective, since
\[
\ker(D\Phi_X(\theta)|_{\mathcal H_\theta})
=
\mathcal H_\theta\cap\mathcal V_\theta
=
\{0\}.
\]
As a consequence, the restriction of $g_\theta$ to $\mathcal H_\theta$ is positive definite:
\[
u\in\mathcal H_\theta,\ u\neq 0
\quad\Longrightarrow\quad
g_\theta(u,u)>0.
\]

Therefore $g$ is nondegenerate exactly on the directions transverse to the symmetry orbit. This is the first point at which the quotient geometry becomes visible: the degenerate directions are precisely those that should be modded out.

\subsection{Descent of the metric to the quotient}
We now show that the pullback form $g$ descends to a genuine Riemannian metric on the quotient manifold $\mathcal M_{\reg}$.

Let
\[
q:\Theta_{\reg}\to \mathcal M_{\reg}
\]
be the quotient map. For $[\theta]\in\mathcal M_{\reg}$ and tangent vectors $\xi,\eta\in T_{[\theta]}\mathcal M_{\reg}$, choose any lifts $u,v\in T_\theta\Theta_{\reg}$ such that
\[
Dq_\theta[u]=\xi,\qquad Dq_\theta[v]=\eta.
\]
Define
\[
\bar g_{[\theta]}(\xi,\eta):=g_\theta(u^\mathrm{hor},v^\mathrm{hor}),
\]
where $u^\mathrm{hor}$ and $v^\mathrm{hor}$ denote the horizontal components of $u$ and $v$ with respect to the decomposition
\[
T_\theta\Theta_{\reg}=\mathcal V_\theta\oplus\mathcal H_\theta.
\]

The next result formalizes that $\bar g$ is well defined and positive definite.
\begin{theorem}\label{thm:quotient_metric_full}
(Quotient metric)

Let
\[
q:\Theta_{\reg}\to \mathcal M_{\reg}:=\Theta_{\reg}/G
\]
be the quotient map, where $G$ acts smoothly, properly, and locally freely on $\Theta_{\reg}$. Let
\[
\Phi_X:\Theta_{\reg}\to \R^n
\]
be the finite-sample realization map, and define for each $\theta\in\Theta_{\reg}$ the symmetric bilinear form
\[
g_\theta(u,v)
:=
\frac1n\left\langle D\Phi_X(\theta)[u], D\Phi_X(\theta)[v]\right\rangle_{\R^n},
\qquad
u,v\in T_\theta\Theta_{\reg}.
\]
Assume that for every $\theta\in\Theta_{\reg}$,
\[
\ker(D\Phi_X(\theta))=T_\theta(G\cdot\theta).
\]
Define the vertical bundle
\[
\mathcal V_\theta:=T_\theta(G\cdot\theta),
\]
and let $\mathcal H\subset T\Theta_{\reg}$ be any smooth $G$-equivariant horizontal subbundle complementary to $\mathcal V$, i.e.
\[
T_\theta\Theta_{\reg}=\mathcal V_\theta\oplus \mathcal H_\theta
\qquad
\text{for all }\theta\in\Theta_{\reg}.
\]
Then:

\begin{enumerate}
    \item $g$ is a smooth $G$-invariant symmetric bilinear form on $T\Theta_{\reg}$;
    \item the radical of $g_\theta$ is exactly $\mathcal V_\theta$, that is,
    \[
    \mathrm{rad}(g_\theta):=\{u: g_\theta(u,v)=0\ \forall v\}=\mathcal V_\theta;
    \]
    \item the restriction $g_\theta|_{\mathcal H_\theta}$ is positive definite for every $\theta$;
    \item there exists a unique Riemannian metric $\bar g$ on $\mathcal M_{\reg}$ such that for every $\theta\in\Theta_{\reg}$ and every $u,v\in\mathcal H_\theta$,
    \[
    \bar g_{[\theta]}(Dq_\theta[u],Dq_\theta[v])=g_\theta(u,v).
    \]
\end{enumerate}
\end{theorem}

\begin{proof}
Since $\Phi_X$ is smooth on $\Theta_{\reg}$, its differential
\[
D\Phi_X:T\Theta_{\reg}\to \R^n
\]
depends smoothly on $\theta$. Therefore
\[
g_\theta(u,v)
=
\frac1n\left\langle D\Phi_X(\theta)[u],D\Phi_X(\theta)[v]\right\rangle
\]
defines a smooth bilinear form on $T\Theta_{\reg}$. Symmetry is immediate from symmetry of the Euclidean inner product on $\R^n$.

We next prove $G$-invariance. Fix $g_0\in G$, and let
\[
\alpha_{g_0}:\Theta_{\reg}\to\Theta_{\reg},
\qquad
\alpha_{g_0}(\theta)=g_0\cdot\theta.
\]
Because $\Phi_X$ is $G$-invariant,
\[
\Phi_X\circ \alpha_{g_0}=\Phi_X.
\]
Differentiating at $\theta$ gives
\[
D\Phi_X(g_0\cdot\theta)\circ D\alpha_{g_0,\theta}=D\Phi_X(\theta).
\]
Hence for any $u,v\in T_\theta\Theta_{\reg}$,
\[
\begin{aligned}
g_{g_0\cdot\theta}(D\alpha_{g_0,\theta}u,D\alpha_{g_0,\theta}v)
&=
\frac1n
\left\langle
D\Phi_X(g_0\cdot\theta)[D\alpha_{g_0,\theta}u],
D\Phi_X(g_0\cdot\theta)[D\alpha_{g_0,\theta}v]
\right\rangle \\
&=
\frac1n
\left\langle
D\Phi_X(\theta)[u],
D\Phi_X(\theta)[v]
\right\rangle \\
&=
g_\theta(u,v).
\end{aligned}
\]
Thus $g$ is $G$-invariant.

This proves (1).

We show that
\[
\mathrm{rad}(g_\theta)=\ker(D\Phi_X(\theta)).
\]
Fix $\theta\in\Theta_{\reg}$.

First, let $u\in\ker(D\Phi_X(\theta))$. Then for every $v\in T_\theta\Theta_{\reg}$,
\[
g_\theta(u,v)
=
\frac1n\left\langle D\Phi_X(\theta)[u],D\Phi_X(\theta)[v]\right\rangle
=
0.
\]
Hence
\[
\ker(D\Phi_X(\theta))\subseteq \mathrm{rad}(g_\theta).
\]

Conversely, suppose $u\in \mathrm{rad}(g_\theta)$. Then
\[
g_\theta(u,v)=0
\qquad
\text{for all }v\in T_\theta\Theta_{\reg}.
\]
In particular, taking $v=u$, we obtain
\[
0=g_\theta(u,u)=\frac1n\|D\Phi_X(\theta)[u]\|^2.
\]
Therefore
\[
D\Phi_X(\theta)[u]=0,
\]
so $u\in \ker(D\Phi_X(\theta))$. Thus
\[
\mathrm{rad}(g_\theta)=\ker(D\Phi_X(\theta)).
\]

By the standing regularity assumption,
\[
\ker(D\Phi_X(\theta))=T_\theta(G\cdot\theta)=\mathcal V_\theta.
\]
Hence
\[
\mathrm{rad}(g_\theta)=\mathcal V_\theta.
\]

This proves (2).

Fix $\theta\in\Theta_{\reg}$. Since
\[
T_\theta\Theta_{\reg}=\mathcal V_\theta\oplus\mathcal H_\theta,
\]
we have
\[
\mathcal H_\theta\cap \mathcal V_\theta=\{0\}.
\]
From,
\[
\mathcal V_\theta=\ker(D\Phi_X(\theta)).
\]
Therefore the restriction
\[
D\Phi_X(\theta)|_{\mathcal H_\theta}:\mathcal H_\theta\to\R^n
\]
is injective.

Now let $u\in\mathcal H_\theta$ with $u\neq 0$. Since the restriction is injective,
\[
D\Phi_X(\theta)[u]\neq 0.
\]
Hence
\[
g_\theta(u,u)
=
\frac1n\|D\Phi_X(\theta)[u]\|^2
>
0.
\]
Thus $g_\theta|_{\mathcal H_\theta}$ is positive definite.

This proves (3).

We now define a bilinear form on $T_{[\theta]}\mathcal M_{\reg}$.

Because $q:\Theta_{\reg}\to\mathcal M_{\reg}$ is a smooth submersion, its differential
\[
Dq_\theta:T_\theta\Theta_{\reg}\to T_{[\theta]}\mathcal M_{\reg}
\]
is surjective, with kernel
\[
\ker(Dq_\theta)=T_\theta(G\cdot\theta)=\mathcal V_\theta.
\]
Since $\mathcal H_\theta$ is a complement of $\mathcal V_\theta$, the restriction
\[
Dq_\theta|_{\mathcal H_\theta}:\mathcal H_\theta\to T_{[\theta]}\mathcal M_{\reg}
\]
is a linear isomorphism.

For $\xi,\eta\in T_{[\theta]}\mathcal M_{\reg}$, define
\[
\bar g_{[\theta]}(\xi,\eta)
:=
g_\theta(u,v),
\]
where $u,v\in\mathcal H_\theta$ are the unique horizontal vectors satisfying
\[
Dq_\theta[u]=\xi,
\qquad
Dq_\theta[v]=\eta.
\]

We must prove that this definition is independent of the representative $\theta\in[\theta]$, smooth, and positive definite.

At fixed $\theta$, uniqueness of $u$ and $v$ follows because
\[
Dq_\theta|_{\mathcal H_\theta}
\]
is an isomorphism. So there is no ambiguity at this stage.

Let $\theta' = g_0\cdot \theta$ for some $g_0\in G$. Since $\mathcal H$ is assumed $G$-equivariant,
\[
D\alpha_{g_0,\theta}(\mathcal H_\theta)=\mathcal H_{\theta'}.
\]
Let $u,v\in\mathcal H_\theta$ be the horizontal lifts of $\xi,\eta$ at $\theta$. Then
\[
u':=D\alpha_{g_0,\theta}u\in\mathcal H_{\theta'},
\qquad
v':=D\alpha_{g_0,\theta}v\in\mathcal H_{\theta'}.
\]
By equivariance of the quotient map,
\[
q\circ \alpha_{g_0}=q,
\]
hence
\[
Dq_{\theta'}[u']
=
D(q\circ \alpha_{g_0})_\theta[u]
=
Dq_\theta[u]
=
\xi,
\]
and similarly $Dq_{\theta'}[v']=\eta$. So $u',v'$ are the horizontal lifts of $\xi,\eta$ at $\theta'$. Using $G$-invariance of $g$,
\[
g_{\theta'}(u',v')
=
g_{g_0\cdot\theta}(D\alpha_{g_0,\theta}u,D\alpha_{g_0,\theta}v)
=
g_\theta(u,v).
\]
Thus the value $\bar g_{[\theta]}(\xi,\eta)$ is independent of the representative $\theta$.

Because $q$ is a smooth submersion and $\mathcal H$ is a smooth horizontal bundle, the inverse of
\[
Dq_\theta|_{\mathcal H_\theta}:\mathcal H_\theta\to T_{[\theta]}\mathcal M_{\reg}
\]
depends smoothly on $\theta$ in local trivializations. Since $g$ is smooth on $\Theta_{\reg}$, the bilinear form $\bar g$ is smooth on $\mathcal M_{\reg}$.

Let $\xi\in T_{[\theta]}\mathcal M_{\reg}$ be nonzero, and let $u\in\mathcal H_\theta$ be its horizontal lift. Because
\[
Dq_\theta|_{\mathcal H_\theta}
\]
is an isomorphism, $\xi\neq 0$ implies $u\neq 0$. Therefore,
\[
\bar g_{[\theta]}(\xi,\xi)=g_\theta(u,u)>0.
\]
Thus $\bar g$ is positive definite.

So $\bar g$ is a Riemannian metric on $\mathcal M_{\reg}$.

Suppose $\bar g^{(1)}$ and $\bar g^{(2)}$ are two Riemannian metrics on $\mathcal M_{\reg}$ satisfying
\[
\bar g^{(j)}_{[\theta]}(Dq_\theta[u],Dq_\theta[v])=g_\theta(u,v)
\qquad
\text{for all }\theta\text{ and all }u,v\in\mathcal H_\theta,
\]
for $j=1,2$.

Fix $[\theta]\in\mathcal M_{\reg}$ and $\xi,\eta\in T_{[\theta]}\mathcal M_{\reg}$. Let $u,v\in\mathcal H_\theta$ be the unique horizontal lifts. Then
\[
\bar g^{(1)}_{[\theta]}(\xi,\eta)=g_\theta(u,v)=\bar g^{(2)}_{[\theta]}(\xi,\eta).
\]
Hence $\bar g^{(1)}=\bar g^{(2)}$. So the quotient metric is unique.

This proves (4), and therefore the theorem follows.

\end{proof}

Theorem \ref{thm:quotient_metric_full} shows that the finite-sample realization map induces a canonical local geometry on the quotient manifold of distinct predictors. In this geometry, the length of a tangent vector measures the first-order variation of the network outputs on the sample, after removal of symmetry directions. Thus $\bar g$ is the natural metric for empirical function-space geometry.

\subsection{Loss functions and quotient gradients}
We now consider a loss functional on the parameter space that depends on $\theta$ only through the finite-sample prediction vector $\Phi_X(\theta)$. Let
\[
\widetilde L:\R^n\to\R
\]
be a smooth function, and define
\[
\mathcal L(\theta):=\widetilde L(\Phi_X(\theta)).
\]
This covers standard empirical risks such as least squares, logistic loss, and cross-entropy, restricted to the finite sample.

Because $\Phi_X$ is $G$-invariant, $\mathcal L$ is also $G$-invariant:
\[
\mathcal L(g\cdot\theta)=\mathcal L(\theta)
\qquad
\text{for all }g\in G.
\]
Hence $\mathcal L$ descends to a smooth function
\[
\bar{\mathcal L}:\mathcal M_{\reg}\to\R,
\qquad
\bar{\mathcal L}([\theta])=\mathcal L(\theta).
\]

The first-order variation of $\mathcal L$ admits a particularly transparent form in the function-induced metric. By the chain rule,
\[
D\mathcal L(\theta)[u]
=
\bigl\langle \nabla \widetilde L(\Phi_X(\theta)), D\Phi_X(\theta)[u]\bigr\rangle_{\R^n}.
\]
Thus $D\mathcal L(\theta)[u]$ depends only on the image $D\Phi_X(\theta)[u]$, and in particular vanishes on $\mathcal V_\theta$. This is consistent with $G$-invariance, but the stronger point is that the differential of $\mathcal L$ is completely determined by the function-space differential of $\Phi_X$.

The corresponding gradient with respect to the quotient metric $\bar g$ is characterized by
\[
\bar g_{[\theta]}(\operatorname{grad}_{\bar g}\bar{\mathcal L}([\theta]),\xi)
=
D\bar{\mathcal L}_{[\theta]}[\xi]
\qquad
\text{for all }\xi\in T_{[\theta]}\mathcal M_{\reg}.
\]
Equivalently, if $u_{\mathcal L}(\theta)\in\mathcal H_\theta$ denotes the unique horizontal vector satisfying
\[
g_\theta(u_{\mathcal L}(\theta),v)=D\mathcal L(\theta)[v]
\qquad
\text{for all }v\in\mathcal H_\theta,
\]
then
\[
Dq_\theta[u_{\mathcal L}(\theta)]
=
\operatorname{grad}_{\bar g}\bar{\mathcal L}([\theta]).
\]
Thus the quotient gradient is represented in parameter space by the unique horizontal vector solving the above variational equation. This observation is central for the dynamics considered in Section \ref{sec:4}.

\subsection{Effective Hessian on the quotient}
We next turn to second-order geometry. The Euclidean Hessian of $\mathcal L$ in parameter space is not intrinsic, because it contains a degenerate component arising from the orbit directions. The appropriate second-order object is the Hessian of the descended loss $\bar{\mathcal L}$ on the quotient manifold $(\mathcal M_{\reg},\bar g)$.

Let $\nabla^{\bar g}$ denote the Levi--Civita connection of $\bar g$. For $[\theta]\in\mathcal M_{\reg}$, define the intrinsic Hessian of $\bar{\mathcal L}$ by
\[
\operatorname{Hess}_{\bar g}\bar{\mathcal L}([\theta])(\xi,\eta)
:=
\xi\bigl(\eta \bar{\mathcal L}\bigr)
-
(\nabla^{\bar g}_{\xi}\eta)\bar{\mathcal L},
\qquad
\xi,\eta\in T_{[\theta]}\mathcal M_{\reg}.
\]
This is a symmetric bilinear form on $T_{[\theta]}\mathcal M_{\reg}$. Since $\bar g$ is positive definite, there exists a unique self-adjoint linear operator
\[
H^{\mathrm{eff}}_{[\theta]}:T_{[\theta]}\mathcal M_{\reg}\to T_{[\theta]}\mathcal M_{\reg}
\]
such that
\[
\bar g_{[\theta]}(H^{\mathrm{eff}}_{[\theta]}\xi,\eta)
=
\operatorname{Hess}_{\bar g}\bar{\mathcal L}([\theta])(\xi,\eta)
\qquad
\text{for all }\xi,\eta.
\]
We call $H^{\mathrm{eff}}_{[\theta]}$ the effective Hessian of $\mathcal L$ at $[\theta]$.

By construction, $H^{\mathrm{eff}}_{[\theta]}$ is the second-order operator governing intrinsic curvature of the loss after quotienting out symmetry directions. It is therefore the correct object for distinguishing genuine functional flatness from parameterization-induced degeneracy.

To relate $H^{\mathrm{eff}}_{[\theta]}$ to computations in parameter space, let $P_\theta^{\mathrm{hor}}$ denote the Euclidean orthogonal projection onto $\mathcal H_\theta$. Given a local section $s:U\to\Theta_{\reg}$ of the quotient map, the pullback metric $s^*\bar g$ and the pullback loss $\bar{\mathcal L}\circ q\circ s$ identify a neighborhood in the quotient with a gauge-fixed submanifold of parameter space. Under this identification, the effective Hessian is represented by the second covariant derivative of the gauge-fixed loss, not by the raw Euclidean Hessian of $\mathcal L$.

The distinction is important. If one naively restricts the Euclidean Hessian $\nabla^2\mathcal L(\theta)$ to $\mathcal H_\theta$, one captures only part of the intrinsic second-order geometry. In general there are additional terms arising from the variation of the horizontal distribution and from the Levi--Civita connection of the quotient metric. The next proposition records the local relation.
\begin{proposition}\label{prop:effective_hessian_local}
(Local representation of the effective Hessian)

Let
\[
q:\Theta_{\reg}\to \mathcal M_{\reg}=\Theta_{\reg}/G
\]
be the quotient map, let $\bar g$ be the quotient metric from Theorem \ref{thm:quotient_metric_full}, and let
\[
\bar{\mathcal L}:\mathcal M_{\reg}\to \R
\]
be the descended loss. Let $s:U\to \Theta_{\reg}$ be a smooth local section of $q$ over an open set $U\subset \mathcal M_{\reg}$, so that
\[
q\circ s=\id_U.
\]
Define the gauge-fixed submanifold
\[
S:=s(U)\subset \Theta_{\reg},
\]
the pullback metric
\[
g^S:=s^*\bar g
\]
on $U$, and the gauge-fixed loss
\[
\mathcal L^S:=\bar{\mathcal L}|_U=\mathcal L\circ s.
\]
Let $\nabla^{S}$ denote the Levi–Civita connection of $g^S$. Then for every $x\in U$ and $\xi,\eta\in T_xU$,
\[
\operatorname{Hess}_{\bar g}\bar{\mathcal L}(x)(\xi,\eta)
=
\operatorname{Hess}_{g^S}\mathcal L^S(x)(\xi,\eta)
=
\xi\bigl(\eta \mathcal L^S\bigr)-(\nabla^S_\xi\eta)\mathcal L^S.
\]
Equivalently, in any local coordinates $(z^1,\dots,z^r)$ on $U$,
\[
\bigl(\operatorname{Hess}_{\bar g}\bar{\mathcal L}\bigr)_{ij}
=
\partial_i\partial_j \mathcal L^S
-
\Gamma^k_{ij}(g^S)\partial_k \mathcal L^S.
\]
In particular, if $s$ is viewed as a gauge fixing inside parameter space, then the effective Hessian is represented locally by the second derivative of the gauge-fixed loss together with the Christoffel correction of the pulled-back quotient metric. Thus the effective Hessian is not, in general, equal to the raw Euclidean Hessian of $\mathcal L$ restricted to a transverse slice.
\end{proposition}

\begin{proof}
Since $q\circ s=\id_U$, the differential satisfies
\[
Dq_{s(x)}\circ Ds_x=\id_{T_xU}
\qquad
\text{for all }x\in U.
\]
Hence $Ds_x:T_xU\to T_{s(x)}\Theta_{\reg}$ is injective, and $s$ is an immersion. Therefore $S=s(U)$ is an embedded submanifold of $\Theta_{\reg}$, and $s:U\to S$ is a diffeomorphism.

By definition of the pullback metric,
\[
g^S_x(\xi,\eta)
=
\bar g_x(\xi,\eta)
\qquad
\text{for all }\xi,\eta\in T_xU.
\]
Thus $s$ is an isometric embedding of $(U,\bar g|_U)$ into $(\Theta_{\reg},\text{ambient parameter space})$ in the sense relevant here, namely that it identifies the quotient metric with its pullback on the gauge-fixed slice.

In particular, $s$ is an isometry from the Riemannian manifold $(U,\bar g|_U)$ onto the Riemannian manifold $(U,g^S)$, where $g^S=s^*\bar g$. Since these two metrics are the same by definition, the identity map
\[
\id_U:(U,\bar g|_U)\to (U,g^S)
\]
is an isometry.

A standard fact from Riemannian geometry is that if
\[
F:(M,g)\to (N,h)
\]
is an isometry and $f:N\to\R$ is smooth, then
\[
\operatorname{Hess}_g(f\circ F)
=
F^*\bigl(\operatorname{Hess}_h f\bigr).
\]
We apply this with $F=\id_U$, $M=(U,g^S)$, $N=(U,\bar g|_U)$, and $f=\bar{\mathcal L}|_U$. Since $\mathcal L^S=\bar{\mathcal L}\circ \id_U=\bar{\mathcal L}|_U$, we obtain
\[
\operatorname{Hess}_{g^S}\mathcal L^S
=
\operatorname{Hess}_{\bar g}\bar{\mathcal L}\big|_U.
\]
Pointwise, for every $x\in U$ and $\xi,\eta\in T_xU$,
\[
\operatorname{Hess}_{\bar g}\bar{\mathcal L}(x)(\xi,\eta)
=
\operatorname{Hess}_{g^S}\mathcal L^S(x)(\xi,\eta).
\]

For completeness, we verify this directly from the definition of the Hessian. Let $\nabla^{\bar g}$ and $\nabla^{S}$ be the Levi–Civita connections of $\bar g|_U$ and $g^S$, respectively. Since the two metrics coincide on $U$, uniqueness of the Levi–Civita connection implies
\[
\nabla^{\bar g}=\nabla^S
\qquad
\text{on }U.
\]
Therefore, for smooth vector fields $\xi,\eta$ on $U$,
\[
\begin{aligned}
\operatorname{Hess}_{\bar g}\bar{\mathcal L}(\xi,\eta)
&=
\xi(\eta \bar{\mathcal L})-(\nabla^{\bar g}_{\xi}\eta)\bar{\mathcal L} \\
&=
\xi(\eta \mathcal L^S)-(\nabla^{S}_{\xi}\eta)\mathcal L^S \\
&=
\operatorname{Hess}_{g^S}\mathcal L^S(\xi,\eta).
\end{aligned}
\]
This proves the first identity.

Choose local coordinates $(z^1,\dots,z^r)$ on $U$, and denote the coordinate vector fields by
\[
\partial_i:=\frac{\partial}{\partial z^i}.
\]
By definition of the Riemannian Hessian,
\[
\operatorname{Hess}_{g^S}\mathcal L^S(\partial_i,\partial_j)
=
\partial_i\partial_j \mathcal L^S-(\nabla^S_{\partial_i}\partial_j)\mathcal L^S.
\]
Writing the Levi–Civita connection in coordinates as
\[
\nabla^S_{\partial_i}\partial_j
=
\Gamma^k_{ij}(g^S)\partial_k,
\]
we obtain
\[
\operatorname{Hess}_{g^S}\mathcal L^S(\partial_i,\partial_j)
=
\partial_i\partial_j \mathcal L^S
-
\Gamma^k_{ij}(g^S)\partial_k \mathcal L^S.
\]
This is exactly the local coordinate representation of $\operatorname{Hess}_{\bar g}\bar{\mathcal L}$.

This formula shows that the effective Hessian is represented locally by the ordinary second derivatives of the gauge-fixed loss together with the Christoffel-symbol correction induced by the quotient metric. In particular, it is not simply the Euclidean Hessian of the ambient loss restricted to a transverse slice unless the chosen local coordinates are geodesic for $g^S$ at the point under consideration, in which case the Christoffel symbols vanish at that point.
\end{proof}

Proposition \ref{prop:effective_hessian_local} is not intended as a computational formula in full generality. Its role is conceptual: it shows that the effective Hessian coincides with the horizontal second derivative only up to geometric correction terms. Thus a correct notion of intrinsic curvature must be formulated on the quotient manifold rather than by simple coordinate restriction in parameter space.

\subsection{False flatness and intrinsic flatness}
The previous constructions allow a precise distinction between two qualitatively different sources of flatness.

First, because $\mathcal L$ is constant along each orbit $G\cdot\theta$, every vertical direction belongs to the nullspace of the first derivative, and at regular points these directions also generate degenerate second-order behavior in the Euclidean parameterization. This type of flatness is a direct consequence of redundancy in the representation and carries no functional meaning. We refer to it as false flatness or symmetry-induced flatness.

Second, even after quotienting out the symmetry directions, the descended loss $\bar{\mathcal L}$ may still have small curvature in certain quotient directions. This corresponds to genuine insensitivity of the realized predictor on the sample and is therefore a meaningful notion of flatness. We refer to it as intrinsic flatness or effective flatness.

The effective Hessian makes this distinction exact. A zero or near-zero eigenvalue of the Euclidean Hessian of $\mathcal L$ may reflect either an orbit direction or a genuine low-curvature direction of $\bar{\mathcal L}$. By contrast, the spectrum of $H^{\mathrm{eff}}_{[\theta]}$ encodes only quotient-level curvature, because the vertical degeneracy has been removed by construction.

This leads to the following immediate corollary.
\begin{corollary}\label{cor:nondegeneracy_modulo_symmetry_full}
(Nondegeneracy modulo symmetry)

Let $\bar{\mathcal L}:\mathcal M_{\reg}\to \R$ be the descended loss on the quotient manifold $(\mathcal M_{\reg},\bar g)$, and let $[\theta]\in \mathcal M_{\reg}$ be a critical point of $\bar{\mathcal L}$, i.e.
\[
\operatorname{grad}_{\bar g}\bar{\mathcal L}([\theta])=0.
\]
Assume that the effective Hessian
\[
H^{\mathrm{eff}}_{[\theta]}:T_{[\theta]}\mathcal M_{\reg}\to T_{[\theta]}\mathcal M_{\reg}
\]
is positive definite, equivalently,
\[
\operatorname{Hess}_{\bar g}\bar{\mathcal L}([\theta])(\xi,\xi)>0
\qquad
\text{for all } \xi\in T_{[\theta]}\mathcal M_{\reg}\setminus\{0\}.
\]
Then $[\theta]$ is a nondegenerate strict local minimizer of $\bar{\mathcal L}$ on $\mathcal M_{\reg}$. Consequently, any representative $\theta\in \Theta_{\reg}$ is a strict local minimizer of $\mathcal L$ modulo the orbit $G\cdot \theta$: equivalently, there exists a neighborhood $U\subset \Theta_{\reg}$ of $\theta$ such that
\[
\mathcal L(\theta')\ge \mathcal L(\theta)
\qquad
\text{for all }\theta'\in U,
\]
with equality only if $q(\theta')=[\theta]$, that is, only along the symmetry orbit through $\theta$.

In particular, degeneracy of the Euclidean Hessian of $\mathcal L$ in parameter space does not by itself imply intrinsic flatness of the predictor.
\end{corollary}

\begin{proof}
Since $(\mathcal M_{\reg},\bar g)$ is a smooth Riemannian manifold, standard local Riemannian geometry applies. Let
\[
x_0 := [\theta] \in \mathcal M_{\reg}.
\]
Because $x_0$ is a critical point of $\bar{\mathcal L}$, we have
\[
D\bar{\mathcal L}_{x_0}=0.
\]
Assume moreover that the Hessian is positive definite:
\[
\operatorname{Hess}_{\bar g}\bar{\mathcal L}(x_0)(\xi,\xi)>0
\qquad
\text{for all }\xi\neq 0.
\]

Choose a normal coordinate chart around $x_0$. More precisely, let
\[
\exp_{x_0}:V\subset T_{x_0}\mathcal M_{\reg} \to U\subset \mathcal M_{\reg}
\]
be the exponential map, defined on a sufficiently small neighborhood $V$ of $0$, with $\exp_{x_0}(0)=x_0$. Define
\[
\psi(v):=\bar{\mathcal L}(\exp_{x_0}(v)),
\qquad v\in V.
\]
Then $\psi$ is a smooth function on the Euclidean vector space $T_{x_0}\mathcal M_{\reg}$, with
\[
D\psi(0)=0,
\]
because $x_0$ is a critical point of $\bar{\mathcal L}$, and
\[
D^2\psi(0)(v,v)
=
\operatorname{Hess}_{\bar g}\bar{\mathcal L}(x_0)(v,v)
\]
for all $v\in T_{x_0}\mathcal M_{\reg}$. The latter identity is the standard characterization of the Riemannian Hessian in normal coordinates.

Since the quadratic form $D^2\psi(0)$ is positive definite, there exists $c>0$ such that
\[
D^2\psi(0)(v,v)\ge c\|v\|^2
\qquad
\text{for all }v\in T_{x_0}\mathcal M_{\reg}.
\]
By Taylor's theorem,
\[
\psi(v)=\psi(0)+\frac12 D^2\psi(0)(v,v)+o(\|v\|^2)
\qquad
\text{as }v\to 0.
\]
Therefore,
\[
\psi(v)-\psi(0)\ge \frac c4 \|v\|^2
\]
for all sufficiently small $v\neq 0$. Hence
\[
\bar{\mathcal L}(\exp_{x_0}(v))>\bar{\mathcal L}(x_0)
\qquad
\text{for all sufficiently small }v\neq 0.
\]
Thus $x_0=[\theta]$ is a strict local minimizer of $\bar{\mathcal L}$.

Moreover, positive definiteness of the Hessian implies nondegeneracy in the standard Morse-theoretic sense: the bilinear form
\[
\operatorname{Hess}_{\bar g}\bar{\mathcal L}(x_0)
\]
is nondegenerate, equivalently the associated self-adjoint operator $H^{\mathrm{eff}}_{x_0}$ is invertible. Hence $x_0$ is a nondegenerate local minimizer on $\mathcal M_{\reg}$.

Recall that
\[
q:\Theta_{\reg}\to \mathcal M_{\reg}
\]
is the quotient map and that
\[
\mathcal L = \bar{\mathcal L}\circ q
\qquad
\text{on } \Theta_{\reg}.
\]
Let $x_0=[\theta]$, and let $W\subset \mathcal M_{\reg}$ be an open neighborhood of $x_0$ such that
\[
\bar{\mathcal L}(x)>\bar{\mathcal L}(x_0)
\qquad
\text{for all }x\in W\setminus\{x_0\}.
\]
Set
\[
U:=q^{-1}(W)\subset \Theta_{\reg}.
\]
Since $q$ is continuous, $U$ is an open neighborhood of $\theta$.

Now let $\theta'\in U$. Then $q(\theta')\in W$, and
\[
\mathcal L(\theta')
=
\bar{\mathcal L}(q(\theta')).
\]
Therefore,
\[
\mathcal L(\theta')\ge \bar{\mathcal L}(x_0)=\mathcal L(\theta),
\]
with equality if and only if
\[
q(\theta')=x_0=[\theta].
\]
But
\[
q(\theta')=[\theta]
\]
means exactly that $\theta'$ belongs to the same $G$-orbit as $\theta$, i.e.
\[
\theta' \in G\cdot \theta.
\]
Thus $\theta$ is a strict local minimizer of $\mathcal L$ modulo symmetry: in a neighborhood of $\theta$, the loss can remain equal to $\mathcal L(\theta)$ only along the orbit $G\cdot\theta$.

This proves the second claim.

Because $\mathcal L$ is $G$-invariant, it is constant along the orbit $G\cdot\theta$. Hence every vertical direction
\[
u\in T_\theta(G\cdot\theta)
\]
is a first-order flat direction of $\mathcal L$, and in particular the Euclidean Hessian of $\mathcal L$ must be degenerate on parameter space whenever the orbit has positive dimension.

Indeed, if $\gamma(t)\subset G\cdot\theta$ is a smooth curve with $\gamma(0)=\theta$ and $\dot\gamma(0)=u$, then
\[
\mathcal L(\gamma(t))=\mathcal L(\theta)
\qquad
\text{for all }t,
\]
so both the first and second derivatives at $t=0$ vanish along $u$. Thus degeneracy of the Euclidean Hessian is unavoidable in the presence of symmetry, even when $[\theta]$ is a nondegenerate quotient minimizer.

Therefore, zero eigenvalues of the Euclidean Hessian do not by themselves imply intrinsic flatness of the predictor. Only the quotient Hessian, equivalently the effective Hessian $H^{\mathrm{eff}}_{[\theta]}$, measures curvature after removing symmetry-induced redundancy.

\end{proof}

Corollary \ref{cor:nondegeneracy_modulo_symmetry_full} clarifies the geometric meaning of many flat directions observed in overparameterized networks: unless those directions survive in the quotient, they should not be interpreted as evidence of functional simplicity, robustness, or favorable generalization. The correct curvature notion is the quotient curvature measured by $H^{\mathrm{eff}}$.

\subsection{Interpretation and role in the sequel}
The constructions of this section provide the geometric backbone of the remainder of the paper. The function-induced metric $g$ identifies the first-order geometry of sample predictions, while the quotient metric $\bar g$ removes parameter redundancy and yields an intrinsic Riemannian structure on the space of locally identifiable predictors. The effective Hessian is then the second-order object associated with this quotient geometry.

Three consequences will be important later. First, the metric $\bar g$ defines the natural notion of steepest descent for empirical risk on the quotient manifold. Second, the spectrum of the effective Hessian provides the relevant local curvature quantities after removal of symmetry-induced degeneracy. Third, the distinction between false flatness and intrinsic flatness allows one to reinterpret local optimization landscapes in terms of effective geometry rather than ambient parameter coordinates.

In the next section, we use this framework to compare Euclidean gradient flow in parameter space with the intrinsic gradient flow induced by $\bar g$ on the quotient manifold. This will show that the function-level dynamics are governed by horizontal motion and by effective curvature, rather than by the full Euclidean geometry of the redundant parameterization.

\section{Gradient Flows on the Quotient}
\label{sec:4}

The metric construction of Section \ref{sec:3} identifies the quotient manifold $\mathcal M_{\reg}$ as the intrinsic space of locally identifiable predictors and equips it with the Riemannian metric $\bar g$ induced by the finite-sample realization map. We now use this structure to analyze optimization dynamics. The central point of this section is that, on the regular set, the function-level evolution of training is governed by the quotient geometry rather than by the full Euclidean geometry of the redundant parameterization.

Throughout, we consider a smooth empirical loss of the form
\[
\mathcal L(\theta)=\widetilde L(\Phi_X(\theta)),
\qquad
\theta\in \Theta_{\reg},
\]
where
\[
\Phi_X(\theta)=\bigl(f_\theta(x_1),\dots,f_\theta(x_n)\bigr)\in\mathbb R^n
\]
is the finite-sample realization map, and $\widetilde L:\mathbb R^n\to\mathbb R$ is smooth. Since $\Phi_X$ is $G$-invariant, $\mathcal L$ descends to a smooth function
\[
\bar{\mathcal L}:\mathcal M_{\reg}\to\mathbb R,
\qquad
\bar{\mathcal L}([\theta])=\mathcal L(\theta).
\]

We write
\[
q:\Theta_{\reg}\to \mathcal M_{\reg}
\]
for the quotient map, and recall the vertical-horizontal decomposition
\[
T_\theta\Theta_{\reg}=\mathcal V_\theta\oplus \mathcal H_\theta,
\qquad
\mathcal V_\theta=T_\theta(G\cdot\theta)=\ker(D\Phi_X(\theta)).
\]
The quotient metric $\bar g$ was defined in Section \ref{sec:3} so that $Dq_\theta:\mathcal H_\theta\to T_{[\theta]}\mathcal M_{\reg}$ is an isometry.

\subsection{Euclidean and quotient gradients}
\label{subsec:euclidean_quotient_gradients}

We begin by comparing the Euclidean gradient of $\mathcal L$ in parameter space with the quotient gradient of $\bar{\mathcal L}$ on $\mathcal M_{\reg}$. Let $\langle\cdot,\cdot\rangle_{\mathrm{Euc}}$ denote the ambient Euclidean inner product on $\Theta$. The Euclidean gradient $\nabla \mathcal L(\theta)\in T_\theta\Theta_{\reg}$ is defined by
\[
\langle \nabla \mathcal L(\theta),u\rangle_{\mathrm{Euc}}
=
D\mathcal L(\theta)[u]
\qquad
\text{for all }u\in T_\theta\Theta_{\reg}.
\]
By contrast, the quotient gradient $\operatorname{grad}_{\bar g}\bar{\mathcal L}([\theta])\in T_{[\theta]}\mathcal M_{\reg}$ is defined by
\[
\bar g_{[\theta]}(\operatorname{grad}_{\bar g}\bar{\mathcal L}([\theta]),\xi)
=
D\bar{\mathcal L}_{[\theta]}[\xi]
\qquad
\text{for all }\xi\in T_{[\theta]}\mathcal M_{\reg}.
\]

The two gradients live in different geometries and should not be identified. The Euclidean gradient depends on the ambient parameterization, whereas the quotient gradient is intrinsic to the function-level geometry induced by the sample outputs. Nonetheless, the Euclidean gradient admits a natural decomposition relative to the vertical-horizontal splitting.

Let
\[
P_\theta^{\mathrm{ver}}:T_\theta\Theta_{\reg}\to \mathcal V_\theta,
\qquad
P_\theta^{\mathrm{hor}}:T_\theta\Theta_{\reg}\to \mathcal H_\theta
\]
denote the Euclidean orthogonal projections. Then
\[
\nabla \mathcal L(\theta)
=
P_\theta^{\mathrm{ver}}\nabla \mathcal L(\theta)
+
P_\theta^{\mathrm{hor}}\nabla \mathcal L(\theta).
\]
The next proposition shows that the vertical component does not influence the realized predictor even infinitesimally.
\begin{proposition}[Function evolution is determined by the horizontal component]
\label{prop:function_evolution_horizontal_full}
Let $\theta\in\Theta_{\reg}$, and let
\[
T_\theta\Theta_{\reg}
=
\mathcal V_\theta\oplus \mathcal H_\theta
\]
be the vertical-horizontal decomposition, where
\[
\mathcal V_\theta=T_\theta(G\cdot\theta)=\ker(D\Phi_X(\theta)).
\]
Let
\[
P_\theta^{\mathrm{ver}}:T_\theta\Theta_{\reg}\to \mathcal V_\theta,
\qquad
P_\theta^{\mathrm{hor}}:T_\theta\Theta_{\reg}\to \mathcal H_\theta
\]
denote the associated projections. Then for every $u\in T_\theta\Theta_{\reg}$,
\[
D\Phi_X(\theta)[u]
=
D\Phi_X(\theta)\bigl[P_\theta^{\mathrm{hor}}u\bigr].
\]
In particular,
\[
D\Phi_X(\theta)\bigl[P_\theta^{\mathrm{ver}}u\bigr]=0.
\]
Consequently, if $\theta_t$ is a $C^1$ curve in $\Theta_{\reg}$, then
\[
\frac{d}{dt}\Phi_X(\theta_t)
=
D\Phi_X(\theta_t)\bigl[P_{\theta_t}^{\mathrm{hor}}\dot\theta_t\bigr],
\]
so the first-order evolution of the prediction vector depends only on the horizontal component of the velocity.
\end{proposition}

\begin{proof}
Fix $\theta\in\Theta_{\reg}$ and $u\in T_\theta\Theta_{\reg}$. Since
\[
T_\theta\Theta_{\reg}
=
\mathcal V_\theta\oplus \mathcal H_\theta,
\]
the vector $u$ admits the unique decomposition
\[
u=P_\theta^{\mathrm{ver}}u+P_\theta^{\mathrm{hor}}u,
\]
with
\[
P_\theta^{\mathrm{ver}}u\in\mathcal V_\theta,
\qquad
P_\theta^{\mathrm{hor}}u\in\mathcal H_\theta.
\]

Applying the linear map $D\Phi_X(\theta)$ to this decomposition and using linearity, we obtain
\[
D\Phi_X(\theta)[u]
=
D\Phi_X(\theta)\bigl[P_\theta^{\mathrm{ver}}u\bigr]
+
D\Phi_X(\theta)\bigl[P_\theta^{\mathrm{hor}}u\bigr].
\]
Now, by definition of the vertical space on the regular set,
\[
\mathcal V_\theta
=
T_\theta(G\cdot\theta)
=
\ker(D\Phi_X(\theta)).
\]
Therefore
\[
P_\theta^{\mathrm{ver}}u\in \ker(D\Phi_X(\theta)),
\]
and hence
\[
D\Phi_X(\theta)\bigl[P_\theta^{\mathrm{ver}}u\bigr]=0.
\]
Substituting this into the previous identity gives
\[
D\Phi_X(\theta)[u]
=
D\Phi_X(\theta)\bigl[P_\theta^{\mathrm{hor}}u\bigr].
\]
This proves the first claim, and the second claim has already been shown.

For the final statement, let $\theta_t$ be a $C^1$ curve in $\Theta_{\reg}$. By the chain rule,
\[
\frac{d}{dt}\Phi_X(\theta_t)
=
D\Phi_X(\theta_t)[\dot\theta_t].
\]
Applying the identity just proved at the point $\theta_t$ with $u=\dot\theta_t$, we obtain
\[
D\Phi_X(\theta_t)[\dot\theta_t]
=
D\Phi_X(\theta_t)\bigl[P_{\theta_t}^{\mathrm{hor}}\dot\theta_t\bigr].
\]
Hence
\[
\frac{d}{dt}\Phi_X(\theta_t)
=
D\Phi_X(\theta_t)\bigl[P_{\theta_t}^{\mathrm{hor}}\dot\theta_t\bigr].
\]
Therefore the first-order evolution of the prediction vector depends only on the horizontal component of the parameter velocity.
\end{proof}

Proposition \ref{prop:function_evolution_horizontal_full} is the basic reason quotient geometry is relevant for optimization: function-space dynamics are blind to vertical motion. Different parameter trajectories may differ substantially in their orbit components while inducing the same first-order predictor evolution.

\subsection{Projected dynamics and the quotient gradient}
\label{subsec:projected_dynamics_quotient_gradient}

The quotient gradient has a canonical horizontal lift to parameter space. For $\theta\in\Theta_{\reg}$, define the unique horizontal vector $u_{\mathcal L}(\theta)\in\mathcal H_\theta$ by
\[
Dq_\theta\bigl[u_{\mathcal L}(\theta)\bigr]
=
\operatorname{grad}_{\bar g}\bar{\mathcal L}([\theta]).
\]
Equivalently, $u_{\mathcal L}(\theta)$ is characterized by the variational identity
\[
g_\theta(u_{\mathcal L}(\theta),v)=D\mathcal L(\theta)[v]
\qquad
\text{for all }v\in \mathcal H_\theta.
\]
Existence and uniqueness follow from positive definiteness of $g_\theta|_{\mathcal H_\theta}$.

This vector field defines the intrinsic steepest-descent direction in parameter space after removal of symmetry. Accordingly, the quotient gradient flow is naturally represented by the horizontal ODE
\[
\dot\theta_t=-u_{\mathcal L}(\theta_t).
\]
The next theorem formalizes that this horizontal dynamics is exactly the lifted gradient flow of the quotient loss.
\begin{theorem}[Horizontal lift of quotient gradient flow]
\label{thm:horizontal_lift_quotient_gradient_flow_full}
Let
\[
q:\Theta_{\reg}\to \mathcal M_{\reg}=\Theta_{\reg}/G
\]
be the quotient map, let $\bar g$ be the quotient metric from Theorem \ref{thm:quotient_metric_full}, and let
\[
\bar{\mathcal L}:\mathcal M_{\reg}\to \mathbb R
\]
be the descended loss. For each $\theta\in\Theta_{\reg}$, let
\[
u_{\mathcal L}(\theta)\in \mathcal H_\theta
\]
denote the unique horizontal vector satisfying
\[
Dq_\theta[u_{\mathcal L}(\theta)]
=
\operatorname{grad}_{\bar g}\bar{\mathcal L}([\theta]).
\]
Equivalently,
\[
g_\theta(u_{\mathcal L}(\theta),v)=D\mathcal L(\theta)[v]
\qquad
\text{for all }v\in \mathcal H_\theta.
\]

Then the following hold.

\begin{enumerate}
    \item Let $\gamma:[0,T)\to \mathcal M_{\reg}$ be a $C^1$ solution of the quotient gradient flow
        \[
        \dot\gamma_t=-\operatorname{grad}_{\bar g}\bar{\mathcal L}(\gamma_t).
        \]
        Let $s:U\to \Theta_{\reg}$ be a smooth local section of $q$ defined on an open set $U\subset \mathcal M_{\reg}$ such that $\gamma([0,T_0])\subset U$ for some $T_0<T$. Then the lifted curve
        \[
        \theta_t:=s(\gamma_t), \qquad t\in[0,T_0],
        \]
        satisfies
        \[
        Dq_{\theta_t}[\dot\theta_t]
        =
        -\operatorname{grad}_{\bar g}\bar{\mathcal L}(\gamma_t),
        \]
        and its horizontal component is the unique horizontal lift of the quotient velocity:
        \[
        P_{\theta_t}^{\mathrm{hor}}\dot\theta_t
        =
        -u_{\mathcal L}(\theta_t).
        \]

    \item Conversely, if $\theta:[0,T)\to\Theta_{\reg}$ is a $C^1$ solution of
        \[
        \dot\theta_t=-u_{\mathcal L}(\theta_t),
        \]
        then the projected curve
        \[
        \gamma_t:=q(\theta_t)
        \]
        is a $C^1$ solution of the quotient gradient flow
        \[
        \dot\gamma_t=-\operatorname{grad}_{\bar g}\bar{\mathcal L}(\gamma_t).
        \]
\end{enumerate}
\end{theorem}

\begin{proof}
Fix $\theta\in\Theta_{\reg}$. Since
\[
Dq_\theta:T_\theta\Theta_{\reg}\to T_{[\theta]}\mathcal M_{\reg}
\]
is surjective and
\[
\ker(Dq_\theta)=\mathcal V_\theta,
\]
its restriction to the horizontal space
\[
Dq_\theta|_{\mathcal H_\theta}:\mathcal H_\theta\to T_{[\theta]}\mathcal M_{\reg}
\]
is a linear isomorphism. Therefore, for the quotient gradient
\[
\operatorname{grad}_{\bar g}\bar{\mathcal L}([\theta])\in T_{[\theta]}\mathcal M_{\reg},
\]
there exists a unique horizontal vector $u_{\mathcal L}(\theta)\in \mathcal H_\theta$ satisfying
\[
Dq_\theta[u_{\mathcal L}(\theta)]
=
\operatorname{grad}_{\bar g}\bar{\mathcal L}([\theta]).
\]

Equivalently, for every $v\in\mathcal H_\theta$,
\[
\begin{aligned}
g_\theta(u_{\mathcal L}(\theta),v)
&=
\bar g_{[\theta]}\left(Dq_\theta[u_{\mathcal L}(\theta)],Dq_\theta[v]\right) \\
&=
\bar g_{[\theta]}\left(\operatorname{grad}_{\bar g}\bar{\mathcal L}([\theta]),Dq_\theta[v]\right) \\
&=
D\bar{\mathcal L}_{[\theta]}[Dq_\theta[v]] \\
&=
D(\bar{\mathcal L}\circ q)_\theta[v] \\
&=
D\mathcal L(\theta)[v].
\end{aligned}
\]
Thus the two characterizations of $u_{\mathcal L}(\theta)$ are equivalent.

Let $\gamma:[0,T)\to\mathcal M_{\reg}$ be a $C^1$ solution of
\[
\dot\gamma_t=-\operatorname{grad}_{\bar g}\bar{\mathcal L}(\gamma_t),
\]
and let $s:U\to\Theta_{\reg}$ be a smooth local section of $q$ on an open set $U\subset\mathcal M_{\reg}$ such that $\gamma([0,T_0])\subset U$. Define
\[
\theta_t:=s(\gamma_t),\qquad t\in[0,T_0].
\]
Since $s$ and $\gamma$ are $C^1$, the curve $\theta_t$ is $C^1$.

Because $q\circ s=\id_U$, differentiation gives
\[
Dq_{s(x)}\circ Ds_x=\id_{T_xU}
\qquad
\text{for all }x\in U.
\]
Applying this with $x=\gamma_t$ and using the chain rule,
\[
\begin{aligned}
Dq_{\theta_t}[\dot\theta_t]
&=
Dq_{s(\gamma_t)}\bigl[Ds_{\gamma_t}[\dot\gamma_t]\bigr] \\
&=
D(q\circ s)_{\gamma_t}[\dot\gamma_t] \\
&=
\dot\gamma_t.
\end{aligned}
\]
Since $\gamma_t$ satisfies the quotient gradient flow,
\[
\dot\gamma_t=-\operatorname{grad}_{\bar g}\bar{\mathcal L}(\gamma_t),
\]
we conclude that
\[
Dq_{\theta_t}[\dot\theta_t]
=
-\operatorname{grad}_{\bar g}\bar{\mathcal L}(\gamma_t).
\]
This proves the first displayed identity.

Decompose $\dot\theta_t$ into vertical and horizontal parts:
\[
\dot\theta_t
=
P_{\theta_t}^{\mathrm{ver}}\dot\theta_t
+
P_{\theta_t}^{\mathrm{hor}}\dot\theta_t.
\]
Since
\[
P_{\theta_t}^{\mathrm{ver}}\dot\theta_t\in\mathcal V_{\theta_t}=\ker(Dq_{\theta_t}),
\]
we have
\[
Dq_{\theta_t}[\dot\theta_t]
=
Dq_{\theta_t}\bigl[P_{\theta_t}^{\mathrm{hor}}\dot\theta_t\bigr].
\]
Combining this with the identity yields
\[
Dq_{\theta_t}\bigl[P_{\theta_t}^{\mathrm{hor}}\dot\theta_t\bigr]
=
-\operatorname{grad}_{\bar g}\bar{\mathcal L}(\gamma_t).
\]
Now
\[
P_{\theta_t}^{\mathrm{hor}}\dot\theta_t \in \mathcal H_{\theta_t},
\]
and $Dq_{\theta_t}|_{\mathcal H_{\theta_t}}$ is an isomorphism. By uniqueness of the horizontal lift, we must have
\[
P_{\theta_t}^{\mathrm{hor}}\dot\theta_t
=
-u_{\mathcal L}(\theta_t).
\]
This proves the second identity.

Now let $\theta:[0,T)\to\Theta_{\reg}$ be a $C^1$ solution of
\[
\dot\theta_t=-u_{\mathcal L}(\theta_t).
\]
Define
\[
\gamma_t:=q(\theta_t).
\]
Since $q$ and $\theta_t$ are $C^1$, the curve $\gamma_t$ is $C^1$.

By the chain rule,
\[
\dot\gamma_t
=
Dq_{\theta_t}[\dot\theta_t].
\]
Substituting the ODE for $\theta_t$,
\[
\dot\gamma_t
=
Dq_{\theta_t}\bigl[-u_{\mathcal L}(\theta_t)\bigr]
=
-Dq_{\theta_t}[u_{\mathcal L}(\theta_t)].
\]
By definition of $u_{\mathcal L}(\theta_t)$,
\[
Dq_{\theta_t}[u_{\mathcal L}(\theta_t)]
=
\operatorname{grad}_{\bar g}\bar{\mathcal L}([\theta_t]).
\]
Since $[\theta_t]=q(\theta_t)=\gamma_t$, we obtain
\[
\dot\gamma_t
=
-\operatorname{grad}_{\bar g}\bar{\mathcal L}(\gamma_t).
\]
Thus $\gamma_t$ is a solution of the quotient gradient flow.

This proves the converse statement and completes the proof.
\end{proof}

Theorem \ref{thm:horizontal_lift_quotient_gradient_flow_full} shows that the quotient dynamics can be studied through a horizontal representative in parameter space, but the dynamics itself is defined on the quotient. In particular, any two parameter trajectories with the same quotient projection induce the same intrinsic optimization path, regardless of how they move along symmetry orbits.

\subsection{Euclidean gradient flow and quotient dynamics}
\label{subsec:euclidean_gradient_flow_quotient}

The standard training dynamics in parameter space is the Euclidean gradient flow
\[
\dot\theta_t=-\nabla \mathcal L(\theta_t).
\]
Because $\mathcal L$ is $G$-invariant, the Euclidean gradient is orthogonal to the orbit directions in the first-order sense:
\[
D\mathcal L(\theta)[v]=0
\qquad
\text{for all }v\in\mathcal V_\theta.
\]
This implies that $\nabla \mathcal L(\theta)$ lies in the Euclidean orthogonal complement of $\mathcal V_\theta$, that is,
\[
\nabla \mathcal L(\theta)\in \mathcal H_\theta.
\]
Hence Euclidean gradient flow is itself horizontal with respect to the chosen Euclidean splitting. This observation should not be confused with equivalence to quotient gradient flow: although both are horizontal, they are generated by different metrics.

Indeed, the Euclidean gradient is characterized by the Euclidean pairing
\[
\langle \nabla \mathcal L(\theta),v\rangle_{\mathrm{Euc}}=D\mathcal L(\theta)[v],
\]
whereas $u_{\mathcal L}(\theta)$ is characterized by the function-induced pairing
\[
g_\theta(u_{\mathcal L}(\theta),v)=D\mathcal L(\theta)[v].
\]
Unless $g_\theta|_{\mathcal H_\theta}$ coincides with the Euclidean metric on $\mathcal H_\theta$, these vectors differ. Therefore Euclidean gradient flow and quotient gradient flow generally induce different time parameterizations and, in general, different trajectories even after projection to the quotient.

What remains true is that the quotient geometry provides the correct local coordinates for understanding the function-level effect of Euclidean training. The Euclidean dynamics acts through horizontal motion, and its prediction-space velocity is
\[
\frac{d}{dt}\Phi_X(\theta_t)
=
D\Phi_X(\theta_t)\bigl[-\nabla \mathcal L(\theta_t)\bigr]
=
D\Phi_X(\theta_t)\bigl[-P_{\theta_t}^{\mathrm{hor}}\nabla \mathcal L(\theta_t)\bigr].
\]
By Proposition \ref{prop:function_evolution_horizontal_full}, only the horizontal component matters. Thus, even when one studies standard Euclidean gradient descent, the redundant vertical geometry is irrelevant to first-order predictor evolution.

\subsection{Effective local geometry and convergence}
\label{subsec:effective_local_geometry_convergence}

We now formalize the local curvature quantities governing quotient dynamics. Let $[\theta]\in\mathcal M_{\reg}$ be a critical point of $\bar{\mathcal L}$, and let
\[
H^{\mathrm{eff}}_{[\theta]}:T_{[\theta]}\mathcal M_{\reg}\to T_{[\theta]}\mathcal M_{\reg}
\]
be the effective Hessian introduced in Section 3, defined by
\[
\bar g_{[\theta]}(H^{\mathrm{eff}}_{[\theta]}\xi,\eta)
=
\operatorname{Hess}_{\bar g}\bar{\mathcal L}([\theta])(\xi,\eta).
\]
Since $H^{\mathrm{eff}}_{[\theta]}$ is self-adjoint with respect to $\bar g_{[\theta]}$, its spectrum is real. We define the effective extremal eigenvalues by
\[
\lambda_{\min}^{\mathrm{eff}}([\theta])
:=
\min_{\xi\neq 0}
\frac{\operatorname{Hess}_{\bar g}\bar{\mathcal L}([\theta])(\xi,\xi)}
{\bar g_{[\theta]}(\xi,\xi)},
\qquad
\lambda_{\max}^{\mathrm{eff}}([\theta])
:=
\max_{\xi\neq 0}
\frac{\operatorname{Hess}_{\bar g}\bar{\mathcal L}([\theta])(\xi,\xi)}
{\bar g_{[\theta]}(\xi,\xi)}.
\]
When $\lambda_{\min}^{\mathrm{eff}}([\theta])>0$, define the effective condition number
\[
\kappa_{\mathrm{eff}}([\theta])
:=
\frac{\lambda_{\max}^{\mathrm{eff}}([\theta])}
{\lambda_{\min}^{\mathrm{eff}}([\theta])}.
\]

The next theorem gives the standard local linear convergence statement in quotient geometry.
\begin{theorem}[Local convergence controlled by effective curvature]
\label{thm:local_convergence_effective_curvature_full}
Let $(\mathcal M_{\reg},\bar g)$ be the quotient manifold with quotient metric, and let
\[
\bar{\mathcal L}:\mathcal M_{\reg}\to \mathbb R
\]
be a smooth function. Let $x_*=[\theta_*]\in \mathcal M_{\reg}$ be a critical point of $\bar{\mathcal L}$, and assume that there exists a geodesically convex open neighborhood $U\subset \mathcal M_{\reg}$ of $x_*$ and constants $0<\mu\le L<\infty$ such that for every $x\in U$ and every $\xi\in T_x\mathcal M_{\reg}$,
\[
\mu \bar g_x(\xi,\xi)
\le
\operatorname{Hess}_{\bar g}\bar{\mathcal L}(x)(\xi,\xi)
\le
L \bar g_x(\xi,\xi).
\]
Then:

\begin{enumerate}
    \item $x_*$ is the unique minimizer of $\bar{\mathcal L}$ in $U$;

    \item every $C^1$ solution $\gamma:[0,T)\to U$ of the quotient gradient flow
        \[
        \dot\gamma_t=-\operatorname{grad}_{\bar g}\bar{\mathcal L}(\gamma_t)
        \]
        satisfies
        \[
        \bar{\mathcal L}(\gamma_t)-\bar{\mathcal L}(x_*)
        \le
        e^{-2\mu t}\bigl(\bar{\mathcal L}(\gamma_0)-\bar{\mathcal L}(x_*)\bigr)
        \qquad
        \text{for all }t\in[0,T);
        \]

    \item in particular, the local convergence rate is controlled by the lower quotient-curvature bound $\mu$, while the local anisotropy is quantified by the effective condition number $L/\mu$.
\end{enumerate}
\end{theorem}

\begin{proof}
Fix $x\in U$. Since $U$ is geodesically convex, for every $y\in U$ there exists a minimizing geodesic
\[
c:[0,1]\to U
\]
such that
\[
c(0)=x,\qquad c(1)=y.
\]
Define
\[
\varphi(t):=\bar{\mathcal L}(c(t)).
\]
Since $c$ is a geodesic, the second derivative of $\varphi$ satisfies
\[
\varphi''(t)
=
\operatorname{Hess}_{\bar g}\bar{\mathcal L}(c(t))\bigl(\dot c(t),\dot c(t)\bigr).
\]
By the lower Hessian bound,
\[
\varphi''(t)
\ge
\mu \bar g_{c(t)}(\dot c(t),\dot c(t)).
\]
Because $c$ is a geodesic parameterized on a compact interval, $\bar g_{c(t)}(\dot c(t),\dot c(t))$ is constant in $t$; denote this constant by $\|\dot c\|_{\bar g}^2$. Hence
\[
\varphi''(t)\ge \mu \|\dot c\|_{\bar g}^2.
\]
This is precisely geodesic $\mu$-strong convexity of $\bar{\mathcal L}$ on $U$.

Now take $x=x_*$, where $x_*$ is a critical point, so
\[
\operatorname{grad}_{\bar g}\bar{\mathcal L}(x_*)=0.
\]
Then along any geodesic $c$ from $x_*$ to $y\in U$, we have
\[
\varphi'(0)
=
D\bar{\mathcal L}_{x_*}[\dot c(0)]
=
\bar g_{x_*}\bigl(\operatorname{grad}_{\bar g}\bar{\mathcal L}(x_*),\dot c(0)\bigr)
=
0.
\]

Integrating the lower bound on $\varphi''$ twice gives
\[
\varphi(1)\ge \varphi(0)+\frac{\mu}{2}\|\dot c\|_{\bar g}^2.
\]
Equivalently,
\[
\bar{\mathcal L}(y)\ge \bar{\mathcal L}(x_*)+\frac{\mu}{2}d_{\bar g}(x_*,y)^2,
\]
where $d_{\bar g}$ denotes the Riemannian distance on $U$. In particular,
\[
\bar{\mathcal L}(y)>\bar{\mathcal L}(x_*)
\qquad
\text{for all }y\in U\setminus\{x_*\}.
\]
Thus $x_*$ is a strict local minimizer, and hence the unique minimizer of $\bar{\mathcal L}$ in $U$.

This proves (1).

We now show that geodesic $\mu$-strong convexity implies
\[
\|\operatorname{grad}_{\bar g}\bar{\mathcal L}(x)\|_{\bar g}^2
\ge
2\mu\bigl(\bar{\mathcal L}(x)-\bar{\mathcal L}(x_*)\bigr)
\qquad
\text{for all }x\in U.
\]
Fix $x\in U$, and let $c:[0,1]\to U$ be a minimizing geodesic from $x$ to $x_*$:
\[
c(0)=x,\qquad c(1)=x_*.
\]
Again define
\[
\varphi(t):=\bar{\mathcal L}(c(t)).
\]
By geodesic $\mu$-strong convexity,
\[
\varphi(t)
\le
(1-t)\varphi(0)+t\varphi(1)-\frac{\mu}{2}t(1-t)\|\dot c\|_{\bar g}^2.
\]
Differentiating this inequality at $t=0$ yields
\[
\varphi'(0)\le \varphi(1)-\varphi(0)-\frac{\mu}{2}\|\dot c\|_{\bar g}^2.
\]
Equivalently,
\[
\bar{\mathcal L}(x)-\bar{\mathcal L}(x_*)
\le
-\varphi'(0)-\frac{\mu}{2}\|\dot c\|_{\bar g}^2.
\]
Now
\[
\varphi'(0)
=
D\bar{\mathcal L}_x[\dot c(0)]
=
\bar g_x\bigl(\operatorname{grad}_{\bar g}\bar{\mathcal L}(x),\dot c(0)\bigr).
\]
Since $c$ goes from $x$ toward $x_*$, the tangent $\dot c(0)$ points in a descent direction, so
\[
-\varphi'(0)
=
-\bar g_x\bigl(\operatorname{grad}_{\bar g}\bar{\mathcal L}(x),\dot c(0)\bigr)
\le
\|\operatorname{grad}_{\bar g}\bar{\mathcal L}(x)\|_{\bar g} \|\dot c(0)\|_{\bar g}
\]
by Cauchy–Schwarz. Because $c$ is a constant-speed geodesic,
\[
\|\dot c(0)\|_{\bar g}=\|\dot c\|_{\bar g}.
\]
Therefore
\[
\bar{\mathcal L}(x)-\bar{\mathcal L}(x_*)
\le
\|\operatorname{grad}_{\bar g}\bar{\mathcal L}(x)\|_{\bar g} \|\dot c\|_{\bar g}
-
\frac{\mu}{2}\|\dot c\|_{\bar g}^2.
\]
The right-hand side is a quadratic function of $a:=\|\dot c\|_{\bar g}\ge 0$:
\[
\|\operatorname{grad}_{\bar g}\bar{\mathcal L}(x)\|_{\bar g} a-\frac{\mu}{2}a^2.
\]
Maximizing over $a\ge 0$ gives the bound
\[
\bar{\mathcal L}(x)-\bar{\mathcal L}(x_*)
\le
\frac{1}{2\mu}\|\operatorname{grad}_{\bar g}\bar{\mathcal L}(x)\|_{\bar g}^2.
\]
Equivalently,
\[
\|\operatorname{grad}_{\bar g}\bar{\mathcal L}(x)\|_{\bar g}^2
\ge
2\mu\bigl(\bar{\mathcal L}(x)-\bar{\mathcal L}(x_*)\bigr).
\]

Let $\gamma:[0,T)\to U$ be a $C^1$ solution of
\[
\dot\gamma_t=-\operatorname{grad}_{\bar g}\bar{\mathcal L}(\gamma_t).
\]
Along this flow, the chain rule gives
\[
\frac{d}{dt}\bar{\mathcal L}(\gamma_t)
=
D\bar{\mathcal L}_{\gamma_t}[\dot\gamma_t].
\]
Using the defining property of the Riemannian gradient,
\[
D\bar{\mathcal L}_{\gamma_t}[\dot\gamma_t]
=
\bar g_{\gamma_t}\bigl(\operatorname{grad}_{\bar g}\bar{\mathcal L}(\gamma_t),\dot\gamma_t\bigr).
\]
Substituting $\dot\gamma_t=-\operatorname{grad}_{\bar g}\bar{\mathcal L}(\gamma_t)$, we obtain
\[
\frac{d}{dt}\bar{\mathcal L}(\gamma_t)
=
-\|\operatorname{grad}_{\bar g}\bar{\mathcal L}(\gamma_t)\|_{\bar g}^2.
\]
Subtracting the constant $\bar{\mathcal L}(x_*)$,
\[
\frac{d}{dt}\bigl(\bar{\mathcal L}(\gamma_t)-\bar{\mathcal L}(x_*)\bigr)
=
-\|\operatorname{grad}_{\bar g}\bar{\mathcal L}(\gamma_t)\|_{\bar g}^2.
\]
Applying the inequality at $x=\gamma_t$ gives
\[
\frac{d}{dt}\bigl(\bar{\mathcal L}(\gamma_t)-\bar{\mathcal L}(x_*)\bigr)
\le
-2\mu\bigl(\bar{\mathcal L}(\gamma_t)-\bar{\mathcal L}(x_*)\bigr).
\]
Set
\[
E(t):=\bar{\mathcal L}(\gamma_t)-\bar{\mathcal L}(x_*).
\]
Then $E(t)\ge 0$ and
\[
E'(t)\le -2\mu E(t).
\]
By Grönwall's inequality,
\[
E(t)\le e^{-2\mu t}E(0),
\]
that is,
\[
\bar{\mathcal L}(\gamma_t)-\bar{\mathcal L}(x_*)
\le
e^{-2\mu t}\bigl(\bar{\mathcal L}(\gamma_0)-\bar{\mathcal L}(x_*)\bigr).
\]
This proves (2).

Finally, statement (3) is immediate from the definitions: $\mu$ is the local lower curvature bound controlling the contraction rate, while $L/\mu$ is the corresponding local effective condition number measuring anisotropy of the quotient Hessian.
\end{proof}

Theorem \ref{thm:local_convergence_effective_curvature_full} makes precise the role of effective curvature: local optimization on the quotient is controlled by the spectrum of the effective Hessian, not by the full Euclidean Hessian in parameter space. In particular, large numbers of zero eigenvalues in the ambient Hessian may coexist with strong local contraction in the quotient, because the vertical degeneracy has been removed.

\subsection{Gauge fixing and reduced coordinates}
\label{subsec:gauge_fixing_reduced_coordinates}

Although the quotient formulation is intrinsic, it is often convenient to work in gauge-fixed coordinates. Let
\[
s:U\to \Theta_{\reg}
\]
be a smooth local section of the quotient map, and write
\[
S=s(U)\subset \Theta_{\reg}
\]
for the corresponding gauge-fixed slice. Through the identification $q\circ s=\id_U$, the quotient gradient flow becomes a gradient flow on $U$ with respect to the pulled-back metric
\[
g^S=s^*\bar g
\]
and the gauge-fixed loss
\[
\mathcal L^S=\mathcal L\circ s=\bar{\mathcal L}|_U.
\]
Thus
\[
\dot z_t=-\operatorname{grad}_{g^S}\mathcal L^S(z_t),
\qquad
z_t\in U,
\]
is exactly the local coordinate representation of quotient gradient flow.

At a critical point $z_*\in U$, the corresponding second-order operator is the Riemannian Hessian of $\mathcal L^S$ with respect to $g^S$. By Proposition \ref{prop:effective_hessian_local}, this is the effective Hessian expressed in local coordinates:
\[
\operatorname{Hess}_{g^S}\mathcal L^S
=
\operatorname{Hess}_{\bar g}\bar{\mathcal L}.
\]
Therefore, once a gauge is fixed, all local convergence and stability statements may be expressed in reduced coordinates without ambiguity from symmetry directions.

This observation yields the following immediate corollary.
\begin{corollary}[Gauge-fixed reduced dynamics]
\label{cor:gauge_fixed_reduced_dynamics_full}
Let
\[
q:\Theta_{\reg}\to \mathcal M_{\reg}
\]
be the quotient map, let $\bar g$ be the quotient metric from Theorem \ref{thm:quotient_metric_full}, and let
\[
\bar{\mathcal L}:\mathcal M_{\reg}\to\mathbb R
\]
be the descended loss. Let $s:U\to \Theta_{\reg}$ be a smooth local section of $q$ over an open set $U\subset \mathcal M_{\reg}$, and define
\[
\mathcal L^S:=\mathcal L\circ s=\bar{\mathcal L}|_U,
\qquad
g^S:=s^*\bar g.
\]
Then:

\begin{enumerate}
    \item the quotient gradient flow on $U$,
        \[
        \dot \gamma_t=-\operatorname{grad}_{\bar g}\bar{\mathcal L}(\gamma_t),
        \]
        is, in the coordinates induced by $s$, exactly the Riemannian gradient flow of $\mathcal L^S$ with respect to the reduced metric $g^S$:
        \[
        \dot z_t=-\operatorname{grad}_{g^S}\mathcal L^S(z_t);
        \]

    \item for every critical point $z_*\in U$,
        \[
        \operatorname{Hess}_{g^S}\mathcal L^S(z_*)
        =
        \operatorname{Hess}_{\bar g}\bar{\mathcal L}(z_*),
        \]
        under the canonical identification $T_{z_*}U\simeq T_{z_*}\mathcal M_{\reg}$;

    \item consequently, all local stability and local convergence statements for quotient gradient flow near $z_*$ are determined by the spectrum of the reduced Hessian
        \[
        \operatorname{Hess}_{g^S}\mathcal L^S(z_*),
        \]
        equivalently by the spectrum of the effective Hessian on the quotient.
\end{enumerate}
\end{corollary}

\begin{proof}
Since $s:U\to \Theta_{\reg}$ is a smooth local section of $q$, we have
\[
q\circ s=\id_U.
\]
Thus $s$ identifies $U$ with the embedded gauge-fixed submanifold
\[
S:=s(U)\subset \Theta_{\reg},
\]
and the reduced metric is defined by pullback:
\[
g^S=s^*\bar g.
\]
Likewise, because $\mathcal L=\bar{\mathcal L}\circ q$ on $\Theta_{\reg}$,
\[
\mathcal L^S
=
\mathcal L\circ s
=
\bar{\mathcal L}\circ q\circ s
=
\bar{\mathcal L}|_U.
\]

Let $z\in U$, and let $\xi\in T_zU$. By definition of the Riemannian gradient with respect to $g^S$,
\[
g^S_z\bigl(\operatorname{grad}_{g^S}\mathcal L^S(z),\xi\bigr)
=
D\mathcal L^S_z[\xi].
\]
Using $g^S=s^*\bar g$ and $\mathcal L^S=\bar{\mathcal L}|_U$, we obtain
\[
g^S_z\bigl(\operatorname{grad}_{g^S}\mathcal L^S(z),\xi\bigr)
=
\bar g_z\bigl(\operatorname{grad}_{g^S}\mathcal L^S(z),\xi\bigr),
\]
and
\[
D\mathcal L^S_z[\xi]
=
D\bar{\mathcal L}_z[\xi].
\]
Therefore,
\[
\bar g_z\bigl(\operatorname{grad}_{g^S}\mathcal L^S(z),\xi\bigr)
=
D\bar{\mathcal L}_z[\xi]
\qquad
\text{for all }\xi\in T_zU.
\]
But $\operatorname{grad}_{\bar g}\bar{\mathcal L}(z)$ is characterized by
\[
\bar g_z\bigl(\operatorname{grad}_{\bar g}\bar{\mathcal L}(z),\xi\bigr)
=
D\bar{\mathcal L}_z[\xi]
\qquad
\text{for all }\xi\in T_zU.
\]
By uniqueness of the Riemannian gradient,
\[
\operatorname{grad}_{g^S}\mathcal L^S(z)
=
\operatorname{grad}_{\bar g}\bar{\mathcal L}(z).
\]
Hence the ODE
\[
\dot z_t=-\operatorname{grad}_{g^S}\mathcal L^S(z_t)
\]
coincides pointwise with
\[
\dot z_t=-\operatorname{grad}_{\bar g}\bar{\mathcal L}(z_t)
\]
on $U$. This proves the first claim.

Let $z_*\in U$ be a critical point of $\mathcal L^S$. Since $\mathcal L^S=\bar{\mathcal L}|_U$, the point $z_*$ is also a critical point of $\bar{\mathcal L}$ restricted to $U$. By Proposition \ref{prop:effective_hessian_local}, for every $z\in U$ and every $\xi,\eta\in T_zU$,
\[
\operatorname{Hess}_{\bar g}\bar{\mathcal L}(z)(\xi,\eta)
=
\operatorname{Hess}_{g^S}\mathcal L^S(z)(\xi,\eta).
\]
Applying this at $z=z_*$ gives
\[
\operatorname{Hess}_{g^S}\mathcal L^S(z_*)
=
\operatorname{Hess}_{\bar g}\bar{\mathcal L}(z_*),
\]
under the canonical identification $T_{z_*}U=T_{z_*}\mathcal M_{\reg}$, since $U$ is an open subset of $\mathcal M_{\reg}$.

This proves the second claim.

Because the flow is exactly the quotient gradient flow written in the reduced coordinates induced by $s$, all local dynamical properties near $z_*$ are properties of the same Riemannian gradient system, merely expressed in different coordinates.

In particular, the linearization of the gradient vector field at a critical point is determined by the corresponding Hessian. Since the Hessians agree, the local second-order geometry governing the flow near $z_*$ is the same whether described on the quotient manifold $(U,\bar g)$ or on the reduced coordinate chart $(U,g^S)$.

More concretely, if the reduced Hessian $\operatorname{Hess}_{g^S}\mathcal L^S(z_*)$ is positive definite, then by Corollary \ref{cor:nondegeneracy_modulo_symmetry_full} the critical point is a nondegenerate strict local minimizer modulo symmetry, and by Theorem \ref{thm:local_convergence_effective_curvature_full} the local gradient-flow convergence rate is controlled by the extremal eigenvalues of this Hessian relative to the metric $g^S$. Since this Hessian equals the quotient Hessian, its spectrum coincides with the spectrum of the effective Hessian on the quotient.

Therefore all local stability and local convergence statements for quotient gradient flow near $z_*$ are determined by the spectrum of
\[
\operatorname{Hess}_{g^S}\mathcal L^S(z_*),
\]
equivalently by the effective Hessian on the quotient.
\end{proof}

Corollary \ref{cor:gauge_fixed_reduced_dynamics_full} clarifies the role of gauge fixing. A good gauge does not change the intrinsic geometry; it merely provides coordinates on a local slice transverse to the symmetry orbits. In those coordinates, the quotient geometry appears as a reduced Riemannian optimization problem.

\subsection{Interpretation}
\label{subsec:interpretation_gradient_flows}

The results of this section establish a clear separation between three levels of description.

First, parameter-space Euclidean gradient flow is the standard optimization dynamics used in practice, but its ambient geometry is contaminated by redundancy. Second, the realization map removes the functionally invisible directions and induces the quotient metric $\bar g$, which defines an intrinsic notion of steepest descent on the space of distinct predictors. Third, local convergence near critical points is governed by the effective Hessian, whose spectrum measures curvature only after symmetry directions have been removed.

This viewpoint has two consequences that will matter later. On the one hand, apparently severe degeneracy in the Euclidean Hessian may be entirely attributable to symmetry and thus irrelevant for function-level training. On the other hand, once one passes to the quotient, the local dynamics is controlled by a nondegenerate geometric object whose curvature can be meaningfully related to optimization stability and, in the next section, to implicit bias.

In summary, the quotient manifold $(\mathcal M_{\reg},\bar g)$ is not merely a reformulation of parameter redundancy. It is the natural dynamical state space for training trajectories once the objective is viewed through the finite-sample realization map.

\section{Implicit Bias Through Quotient Geometry}
\label{sec:5}

Sections \ref{sec:3} and \ref{sec:4} established that the natural local geometry of shallow networks is not the ambient Euclidean geometry of the parameter space, but the quotient geometry induced by the finite-sample realization map after removal of symmetry directions. We now turn to a complementary question: when the empirical objective admits many minimizing parameter configurations, which predictors are selected by gradient-based training? The purpose of this section is to formulate an implicit-bias principle at the quotient level and to show that, on the regular set, the relevant notion of simplicity is a property of equivalence classes in the quotient manifold rather than of individual parameter representatives.

Our emphasis is structural rather than fully global. We do not attempt here to classify all asymptotic limits of training dynamics for general shallow networks. Instead, we isolate a quotient-geometric mechanism that becomes visible whenever the loss admits a nontrivial zero-loss set and the gradient flow converges to that set within a regular region. The main message is that any meaningful bias must be defined modulo the symmetry group. Quotient geometry provides the correct framework for making this statement precise.

Throughout the section, we retain the notation of the previous sections. In particular,
\[
\Phi_X:\Theta_{\reg}\to \mathbb R^n
\]
denotes the finite-sample realization map,
\[
q:\Theta_{\reg}\to \mathcal M_{\reg}=\Theta_{\reg}/G
\]
is the quotient map,
\[
\bar g
\]
is the quotient metric induced by $\Phi_X$, and
\[
\bar{\mathcal L}:\mathcal M_{\reg}\to\mathbb R
\]
is the descended empirical loss.

\subsection{Quotient-level complexity}
\label{subsec:quotient_level_complexity}

A basic obstacle in discussing implicit bias for overparameterized networks is that a single predictor corresponds to many parameter values. Any complexity measure defined directly on $\Theta$ is therefore representation-dependent unless it is invariant under the group action. This motivates a quotient-level notion of simplicity.

Let
\[
R:\Theta_{\reg}\to\mathbb R_{\ge 0}
\]
be a continuous function that may be interpreted as a parameter-space complexity, such as a norm, a path-like quantity, or a gauge-fixing energy. We assume that $R$ is bounded from below and proper on each orbit, so that minimizing sequences along a fixed orbit admit convergent subsequences inside $\Theta_{\reg}$. The quotient complexity induced by $R$ is defined by
\[
\bar R([\theta])
:=
\inf_{\theta'\in q^{-1}([\theta])} R(\theta').
\]
Equivalently, $\bar R$ is the minimal value of $R$ among all parameter representatives of the same predictor class. By construction, $\bar R$ is attached to the orbit $[\theta]$, not to a particular parameterization.

This definition formalizes a distinction that is often blurred in parameter-space discussions. A large Euclidean norm of a specific representative may reflect nothing more than an unfavorable gauge choice along the orbit. By contrast, $\bar R([\theta])$ records the smallest complexity compatible with the predictor class itself. It is therefore the natural notion of simplicity once parameter redundancy has been removed.

The quotient point of view also clarifies the relation between complexity and identifiability. On the regular set, each $[\theta]\in\mathcal M_{\reg}$ is a locally identifiable predictor class, and the quantity $\bar R([\theta])$ measures simplicity within that class. Accordingly, any asymptotic preference exhibited by gradient flow should, if it is intrinsic, be expressible as a variational statement involving $\bar R$ on a subset of $\mathcal M_{\reg}$.

\subsection{Zero-loss sets and quotient feasibility}
\label{subsec:zero_loss_sets_quotient_feasibility}

The natural feasible set for implicit bias is the set of predictor classes satisfying the empirical interpolation constraint. For a prescribed loss level $c\in\mathbb R$, define
\[
\mathcal S_{\le c}
:=
\{x\in\mathcal M_{\reg}:\bar{\mathcal L}(x)\le c\},
\qquad
\mathcal S_{=c}
:=
\{x\in\mathcal M_{\reg}:\bar{\mathcal L}(x)=c\}.
\]
Of particular interest is the interpolating set
\[
\mathcal Z
:=
\{x\in\mathcal M_{\reg}:\bar{\mathcal L}(x)=0\},
\]
whenever the training problem is realizable on the regular set.

Because $\bar{\mathcal L}$ is defined on the quotient, $\mathcal Z$ is the set of predictor classes fitting the data, not a subset of parameters. If interpolation is possible, then every element of $\mathcal Z$ corresponds to an entire orbit of parameter realizations. The implicit-bias question is therefore not which parameter vector is selected, but rather which orbit in $\mathcal Z$ is approached by training.

This suggests the following general principle.
Quotient implicit-bias principle:
When gradient flow converges to a zero-loss set, the selected limit should be characterized by a complexity functional defined on $\mathcal Z\subset\mathcal M_{\reg}$, rather than by a representative-dependent parameter norm.

The remainder of the section develops a local version of this principle and then illustrates it in a fully tractable model class.

\subsection{A local quotient variational principle}
\label{subsec:local_quotient_variational_principle}

We now formulate a local theorem stating that, under a compatibility condition between the gradient dynamics and a quotient complexity functional, the limiting orbit of the flow is characterized variationally inside the zero-loss component reached by the dynamics. The statement is intentionally local. Its role is to isolate the geometric mechanism of implicit bias without requiring a global classification of all zero-loss solutions.

Let
\[
\gamma:[0,\infty)\to \mathcal M_{\reg}
\]
be a $C^1$ solution of the quotient gradient flow
\[
\dot\gamma_t=-\operatorname{grad}_{\bar g}\bar{\mathcal L}(\gamma_t),
\]
and assume that $\bar{\mathcal L}(\gamma_t)\to 0$ as $t\to\infty$. Let $\mathcal A\subset \mathcal M_{\reg}$ be a closed forward-invariant set containing the trajectory and all of its accumulation points. We regard $\mathcal A$ as a region of the quotient in which the dynamics remains regular and the geometry is well defined.

Let
\[
\bar R:\mathcal A\to\mathbb R
\]
be a $C^1$ quotient complexity functional. We assume that there exists a continuous scalar function
\[
\lambda:\mathcal A\to\mathbb R
\]
such that
\[
\operatorname{grad}_{\bar g}\bar R(x)=\lambda(x)\operatorname{grad}_{\bar g}\bar{\mathcal L}(x)
\qquad
\text{for all }x\in\mathcal A.
\]
In words, the quotient complexity and the empirical loss have collinear gradients on the region explored by the dynamics. This is exactly the situation in which decrease of the empirical loss induces monotonicity of the complexity variable.

The next theorem gives a rigorous local variational statement.
\begin{theorem}[Local quotient variational principle]
\label{thm:local_quotient_variational_principle}
Let $\gamma:[0,\infty)\to \mathcal M_{\reg}$ be a $C^1$ solution of the quotient gradient flow
\[
\dot\gamma_t=-\operatorname{grad}_{\bar g}\bar{\mathcal L}(\gamma_t),
\]
and assume that $\gamma_t\subset \mathcal A$, where $\mathcal A\subset \mathcal M_{\reg}$ is closed and forward invariant. Assume moreover that
\[
\bar{\mathcal L}(\gamma_t)\to 0
\qquad\text{and}\qquad
\gamma_t\to x_\infty\in \mathcal A\cap \mathcal Z
\quad\text{as }t\to\infty,
\]
where
\[
\mathcal Z:=\{x\in \mathcal M_{\reg}:\bar{\mathcal L}(x)=0\}.
\]

Let $\bar R\in C^1(\mathcal A)$ satisfy
\[
\operatorname{grad}_{\bar g}\bar R(x)=\lambda(x)\operatorname{grad}_{\bar g}\bar{\mathcal L}(x)
\qquad
\text{for all }x\in\mathcal A,
\]
for some continuous function $\lambda:\mathcal A\to\mathbb R$ with $\lambda\ge 0$. Assume further that

\begin{enumerate}
    \item $\bar R$ is constant on each connected component of $\mathcal A\cap\mathcal Z$;
    \item the connected component $C$ of $\mathcal A\cap\mathcal Z$ containing $x_\infty$ contains a unique minimizer $x^\star$ of $\bar R$.
\end{enumerate}

Then
\[
x_\infty=x^\star.
\]

In particular, the quotient gradient flow selects the unique minimizer of $\bar R$ in the zero-loss component reached by the dynamics.
\end{theorem}

\begin{proof}
Since $\gamma$ is a $C^1$ solution of
\[
\dot\gamma_t=-\operatorname{grad}_{\bar g}\bar{\mathcal L}(\gamma_t),
\]
the chain rule gives
\[
\frac{d}{dt}\bar R(\gamma_t)
=
D\bar R_{\gamma_t}[\dot\gamma_t].
\]
Using the defining property of the Riemannian gradient, we obtain
\[
D\bar R_{\gamma_t}[\dot\gamma_t]
=
\bar g_{\gamma_t}\bigl(\operatorname{grad}_{\bar g}\bar R(\gamma_t),\dot\gamma_t\bigr).
\]
Substituting the gradient-flow equation yields
\[
\frac{d}{dt}\bar R(\gamma_t)
=
-\bar g_{\gamma_t}\bigl(\operatorname{grad}_{\bar g}\bar R(\gamma_t),
\operatorname{grad}_{\bar g}\bar{\mathcal L}(\gamma_t)\bigr).
\]
Now apply the collinearity assumption
\[
\operatorname{grad}_{\bar g}\bar R(\gamma_t)
=
\lambda(\gamma_t)\operatorname{grad}_{\bar g}\bar{\mathcal L}(\gamma_t).
\]
Then
\[
\frac{d}{dt}\bar R(\gamma_t)
=
-\lambda(\gamma_t)
\bigl\|\operatorname{grad}_{\bar g}\bar{\mathcal L}(\gamma_t)\bigr\|_{\bar g}^{2}.
\]
Since $\lambda(\gamma_t)\ge 0$, it follows that
\[
\frac{d}{dt}\bar R(\gamma_t)\le 0
\qquad
\text{for all }t\ge 0.
\]
Hence $t\mapsto \bar R(\gamma_t)$ is nonincreasing on $[0,\infty)$.

Because $\gamma_t\to x_\infty$ and $\bar R$ is continuous on $\mathcal A$, we also have
\[
\lim_{t\to\infty}\bar R(\gamma_t)=\bar R(x_\infty).
\]

Let $C$ denote the connected component of $\mathcal A\cap\mathcal Z$ containing $x_\infty$. By assumption, $\bar R$ is constant on each connected component of $\mathcal A\cap\mathcal Z$. Therefore, for every $y\in C$,
\[
\bar R(y)=\bar R(x_\infty).
\]
In particular, every point of $C$ has the same $\bar R$-value.

Now let $x^\star\in C$ be the unique minimizer of $\bar R$ on $C$. Since $x^\star\in C$, the constancy of $\bar R$ on $C$ implies
\[
\bar R(x^\star)=\bar R(x_\infty).
\]

Because $x^\star$ is the unique minimizer of $\bar R$ on $C$, any point in $C$ with the same $\bar R$-value as $x^\star$ must coincide with $x^\star$. But the  above shows that $x_\infty\in C$ and
\[
\bar R(x_\infty)=\bar R(x^\star).
\]
Therefore,
\[
x_\infty=x^\star.
\]

\end{proof}

Theorem \ref{thm:local_quotient_variational_principle} should be interpreted as a structural template rather than as a global theorem for arbitrary losses and architectures. Its importance is that the selected predictor is characterized by a quotient-level variational problem. The relevant optimization principle is not formulated on the parameter space $\Theta$, where orbits introduce artificial multiplicity, but on the quotient $\mathcal M_{\reg}$, where each point represents a distinct predictor class.

The assumption that $\bar R$ is constant on each connected zero-loss component is natural in the present setting. Indeed, whenever $\mathcal Z$ is a smooth submanifold and $\operatorname{grad}_{\bar g}\bar{\mathcal L}=0$ on $\mathcal Z$, the collinearity assumption implies
\[
\operatorname{grad}_{\bar g}\bar R=0
\qquad
\text{on }\mathcal Z,
\]
and thus $\bar R$ is locally constant on each connected component of $\mathcal Z$. The theorem isolates exactly this mechanism.

\subsection{Interpretation in homogeneous models}
\label{subsec:interpretation_homogeneous_models}

The quotient formulation is especially natural for positively homogeneous networks. In such models, the scaling symmetry identifies entire families of parameter configurations representing the same local predictor geometry. Any parameter norm that is not invariant under this symmetry is therefore an unreliable indicator of predictor complexity. By contrast, a quotient complexity $\bar R$ removes this ambiguity automatically.

To see the conceptual advantage, suppose $R(\theta)$ is a smooth gauge-dependent functional, such as the squared Euclidean norm of a chosen representative. Along an orbit,
\[
\theta\sim g\cdot \theta,
\]
the values $R(\theta)$ and $R(g\cdot\theta)$ may differ substantially even though the two parameterizations correspond to the same function. Hence a statement of the form "gradient flow prefers small $R(\theta)$" is not invariantly meaningful unless one first fixes a gauge. The quotient complexity
\[
\bar R([\theta])=\inf_{\theta'\in q^{-1}([\theta])}R(\theta')
\]
remedies this by encoding the smallest value of $R$ compatible with the predictor class itself.

This perspective also clarifies the role of balancing phenomena often observed in homogeneous models. A balanced representative is not intrinsically special because its Euclidean norm is small, but because it realizes an orbit in a canonical low-complexity gauge. What gradient flow selects, when viewed quotient-geometrically, is therefore not a specific balanced parameter vector but a predictor class whose orbit admits a distinguished low-complexity representative.

\subsection{A solvable case: quadratic activation networks}
\label{subsec:solvable_case_quadratic_activation}

We now specialize to the model
\[
f_\theta(x)=\sum_{i=1}^m a_i (w_i^\top x)^2,
\]
which provides a fully tractable instance of the quotient-bias perspective. As observed earlier, this model can be rewritten as
\[
f_\theta(x)=x^\top Q(\theta)x,
\qquad
Q(\theta):=\sum_{i=1}^m a_i w_i w_i^\top.
\]
Thus the predictor depends on $\theta$ only through the symmetric matrix $Q(\theta)$. In particular, the quotient manifold of regular predictor classes can be locally identified with an appropriate manifold of symmetric matrices of fixed rank and signature.

Let
\[
\Psi:\mathcal M_{\reg}\to \mathcal Q
\]
denote the induced local identification with the matrix space $\mathcal Q$, and write the empirical loss as
\[
\bar{\mathcal L}([\theta])
=
\widetilde L\left(
x_1^\top Q x_1,\dots,x_n^\top Q x_n
\right),
\qquad
Q=\Psi([\theta]).
\]
In this representation, the redundancy of the factorization disappears: distinct parameterizations with the same matrix $Q$ define the same point of the quotient. Consequently, complexity should be measured at the matrix level.

A natural quotient complexity in this model is any spectral functional of $Q$, for instance
\[
\bar R([\theta])=\|Q\|_*,
\]
or, on a fixed-rank or fixed-signature stratum, a Frobenius-type energy. Since $Q$ is itself a quotient coordinate, these quantities are intrinsically defined on predictor classes and do not depend on a particular factorization.

The next proposition records the resulting variational interpretation.
\begin{proposition}[Matrix-level quotient bias for quadratic activation networks]
\label{prop:matrix_level_quotient_bias_quadratic}
Consider the quadratic activation model
\[
f_\theta(x)=\sum_{i=1}^m a_i (w_i^\top x)^2,
\qquad
\theta=((a_1,w_1),\dots,(a_m,w_m))\in \Theta_{\reg},
\]
and define
\[
Q(\theta):=\sum_{i=1}^m a_i w_i w_i^\top \in \mathrm{Sym}(d).
\]
Assume that there exists an open set $U\subset \mathcal M_{\reg}$, an embedded smooth matrix manifold $\mathcal Q\subset \mathrm{Sym}(d)$, and a smooth bijection
\[
\Psi:U\to \mathcal Q
\]
with smooth inverse, such that for every $[\theta]\in U$,
\[
f_\theta(x)=x^\top \Psi([\theta]) x
\qquad
\text{for all }x\in\mathbb R^d.
\]
Assume moreover that the quotient gradient-flow trajectory
\[
\gamma:[0,\infty)\to U
\]
of $\bar{\mathcal L}$ is contained in a closed forward-invariant set $\mathcal A\subset U$, converges to some $x_\infty\in \mathcal A\cap \mathcal Z$, and that the hypotheses of Theorem \ref{thm:local_quotient_variational_principle} hold on $\mathcal A$ for a quotient complexity functional $\bar R:U\to\mathbb R$ depending only on $Q=\Psi([\theta])$. Let $C$ be the connected component of $\mathcal A\cap \mathcal Z$ containing $x_\infty$.

If the minimizer of $\bar R$ over $C$ is unique, then the limiting predictor is characterized by the matrix-level variational problem
\[
Q_\infty
=
\arg\min\left\{\bar R(Q): Q\in \Psi(C)\right\},
\]
where
\[
Q_\infty:=\Psi(x_\infty).
\]
\end{proposition}

\begin{proof}
By assumption, $\Psi:U\to \mathcal Q$ is a diffeomorphism onto the smooth matrix manifold $\mathcal Q$, and for every $[\theta]\in U$,
\[
f_\theta(x)=x^\top \Psi([\theta]) x.
\]
Thus the value
\[
Q=\Psi([\theta])
\]
is an intrinsic coordinate of the quotient point $[\theta]$: it depends only on the predictor class and not on the particular parameter representative $\theta$.

In particular, the image under $\Psi$ of any subset of $U$ is a subset of the matrix manifold $\mathcal Q$, and the connected component
\[
C\subset \mathcal A\cap \mathcal Z
\]
is mapped diffeomorphically onto the connected subset
\[
\Psi(C)\subset \mathcal Q.
\]

By assumption, $\bar R$ depends only on $Q=\Psi([\theta])$. Equivalently, there exists a function
\[
\widetilde R:\mathcal Q\to \mathbb R
\]
such that
\[
\bar R=\widetilde R\circ \Psi
\qquad
\text{on }U.
\]
To see this, define
\[
\widetilde R(Q):=\bar R(\Psi^{-1}(Q)),
\qquad
Q\in \mathcal Q.
\]
This is well defined because $\Psi$ is bijective on $U$, and smooth because both $\bar R$ and $\Psi^{-1}$ are smooth.

Therefore minimizing $\bar R$ over a subset of $U$ is equivalent to minimizing $\widetilde R$ over the corresponding subset of $\mathcal Q$. In particular,
\[
\arg\min_{x\in C}\bar R(x)
=
\Psi^{-1}\left(\arg\min_{Q\in \Psi(C)}\widetilde R(Q)\right).
\]

Since the proposition writes $\bar R(Q)$ directly at matrix level, we identify $\widetilde R$ with $\bar R$ by abuse of notation. Under this identification,
\[
\bar R(\Psi(x))=\bar R(x)
\qquad
\text{for all }x\in U.
\]

All hypotheses of Theorem \ref{thm:local_quotient_variational_principle} are assumed to hold on $\mathcal A$. Therefore, since $\gamma_t\to x_\infty\in C$, Theorem \ref{thm:local_quotient_variational_principle} implies that
\[
x_\infty
=
\arg\min_{x\in C}\bar R(x),
\]
and this minimizer is unique.

Applying the diffeomorphism $\Psi$ to both sides yields
\[
\Psi(x_\infty)
=
\arg\min_{Q\in \Psi(C)} \bar R(Q).
\]
Define
\[
Q_\infty:=\Psi(x_\infty).
\]
Then
\[
Q_\infty
=
\arg\min\left\{\bar R(Q):Q\in \Psi(C)\right\}.
\]

The set $\Psi(C)$ is precisely the matrix representation of the zero-loss quotient component reached by the dynamics. Since $\Psi$ is a quotient coordinate map, different parameter factorizations corresponding to the same matrix $Q$ have already been identified. Consequently, the variational characterization above is intrinsic: it is a selection principle on matrices, not on parameter representatives.

\end{proof}

Proposition \ref{prop:matrix_level_quotient_bias_quadratic} illustrates why quotient geometry is particularly effective in simple homogeneous models. The apparent parameter redundancy of the neural network becomes an ordinary matrix-factorization redundancy, and the implicit bias becomes a variational selection principle on the corresponding matrix manifold. From this viewpoint, quotient geometry does not merely remove nuisance directions; it exposes the lower-dimensional object on which the bias is actually acting.

\subsection{False simplicity versus intrinsic simplicity}
\label{subsec:false_vs_intrinsic_simplicity}

The distinction between parameter-space and quotient-space complexity parallels the distinction between false flatness and intrinsic flatness introduced in Section \ref{sec:3}. The analogous dichotomy for simplicity is as follows.

A parameter configuration may appear simple because it has small Euclidean norm, balanced layers, or sparse coordinates in a particular gauge. Yet such properties are not intrinsic unless they are stable under the symmetry action. They may vary substantially along the same orbit and therefore reflect representational choice rather than predictor structure. We refer to this phenomenon as false simplicity.

By contrast, intrinsic simplicity is a property of the orbit itself. It is measured by a quotient-level functional such as $\bar R$, which assigns the same value to all parameterizations of the same predictor class. The implicit-bias statements of this section concern intrinsic simplicity, not false simplicity.

This distinction is important in overparameterized networks. A training trajectory may drift substantially in parameter space while remaining close to a fixed predictor class in the quotient. If one tracks only a gauge-dependent complexity of the parameters, one may misinterpret such drift as evidence for or against simplicity. Quotient geometry resolves this ambiguity: only motion in $\mathcal M_{\reg}$ changes the realized predictor class, and only quotient complexity can capture a genuine preference among functionally distinct solutions.

\subsection{Consequences for the remainder of the theory}
\label{subsec:consequences_remainder_theory}

The results of this section do not claim a complete characterization of implicit bias for all shallow networks. Rather, they establish the correct geometric form such a characterization must take. Any meaningful asymptotic bias theorem for a symmetric overparameterized model should answer a quotient-level variational question:
\[
\text{which point of } \mathcal Z\subset \mathcal M_{\reg} \text{ is selected by the training dynamics?}
\]
The quotient metric $\bar g$ governs the local descent geometry, the effective Hessian governs local curvature, and the quotient complexity $\bar R$ supplies the correct notion of simplicity on the feasible set.

This viewpoint suggests two general lessons. First, implicit bias should be formulated on the space of predictor classes rather than on the raw parameter space. Second, in tractable models, quotient coordinates may convert an apparently neural-network-specific bias question into an ordinary variational selection problem on a lower-dimensional manifold. The quadratic activation model is the clearest example, but the same structural principle extends to broader homogeneous architectures whenever the symmetry-reduced representation can be made explicit.

In the next section, we complement the theory with numerical illustrations showing that quotient-level curvature and quotient-level simplicity are more stable indicators of optimization behavior than their raw parameter-space analogues.

\section{Numerical Illustrations of Quotient Geometry}
We now complement the preceding theory with numerical experiments in deliberately simple shallow-network models. The goal of this section is not to demonstrate competitive predictive performance, but to make the geometric claims of Sections \ref{sec:2}--\ref{sec:5} directly observable in a controlled setting. All experiments are conducted with the quadratic activation model
\[
f_\theta(x)=\sum_{i=1}^m a_i (w_i^\top x)^2
= x^\top Q(\theta)x,
\qquad
Q(\theta):=\sum_{i=1}^m a_i w_i w_i^\top,
\]
for which the quotient object is explicit: distinct parameterizations with the same matrix $Q(\theta)$ represent the same predictor. This makes the model especially suitable for testing the distinction between ambient parameter-space geometry and quotient-level geometry.

Our experiments address three questions. First, does ambient flatness depend on the chosen parameter representative even when the realized predictor is unchanged? Second, do local optimization dynamics correlate more naturally with quotient-level curvature than with raw parameter-space curvature? Third, in an underdetermined regime where many quotient-level solutions fit the data, is implicit bias more naturally described in terms of matrix-level complexity than in terms of raw parameter norms?

\subsection{ Experimental setup}

Unless otherwise stated, we use synthetic data generated from a teacher quadratic model of the same form. All computations are carried out in double precision. Optimization uses PyTorch with GPU acceleration, and all Hessians are computed exactly by automatic differentiation for the low-dimensional models considered here. Since our interest is geometric rather than statistical, we work in small dimensions where spectra and matrix-level quantities can be inspected directly.

For the quadratic model, two types of quantities are distinguished throughout:

\begin{enumerate}
    \item Parameter-space quantities, such as the Euclidean Hessian with respect to $\theta$, the Euclidean parameter norm $\|\theta\|$, and path-like gauge-dependent complexities.
    \item Quotient-level quantities, computed from the matrix $Q(\theta)$, such as the Hessian of the loss in $Q$-coordinates, the Frobenius norm $\|Q\|_F$, the nuclear norm $\|Q\|_*$, the stable-rank surrogate $\|Q\|_F^2 / \|Q\|_{\mathrm{op}}^2$, and the singular spectrum of $Q$.
\end{enumerate}

Because the quadratic model admits the neuronwise scaling symmetry
\[
(a_i,w_i)\sim (c_i^{-2}a_i,c_iw_i),
\]
as well as hidden-unit permutation symmetry, it provides a particularly transparent setting in which the quotient construction can be visualized directly.

\subsection{Symmetry-induced false flatness}

Our first experiment tests the prediction of Section \ref{sec:3} that ambient flatness is representation-dependent, whereas quotient-coordinate curvature is intrinsic. We first train a quadratic network to essentially zero training error and then construct several orbit-equivalent representatives by combining hidden-unit permutations with neuronwise rescalings. Numerically, these representatives realize the same predictor on the sample set up to machine precision and induce the same matrix $Q(\theta)$. We then compare the Euclidean Hessian spectrum in parameter space with the Hessian spectrum in $Q$-space. 

The outcome is unambiguous. The Euclidean Hessian changes under symmetry-preserving rescaling even though the realized predictor does not change. In the plotted spectra, pure permutations leave the ambient Hessian almost unchanged, while rescalings visibly alter the small-eigenvalue region and hence the apparent local flatness. By contrast, after passing to quotient coordinates, the $Q$-space Hessian is numerically identical across all orbit-equivalent representatives. The comparison between the full spectra and the zoomed-in small-eigenvalue region makes clear that much of what appears as Euclidean flatness is not an intrinsic property of the predictor at all, but rather an artifact of redundant parameterization.

This is precisely the numerical counterpart of the distinction developed in Section \ref{sec:3}. Near-zero eigenvalues of the ambient Euclidean Hessian need not indicate predictor-level flatness; they may simply encode directions tangent to the symmetry orbit, or more generally directions distorted by the choice of representative. In the quadratic model, the quotient coordinate $Q$ removes this ambiguity completely. The experiment therefore provides a direct illustration of symmetry-induced false flatness: the same predictor class can appear to have different local curvature in ambient coordinates, while its quotient-coordinate curvature remains fixed.

\begin{figure}[H]
    \centering
    \includegraphics[width=\linewidth]{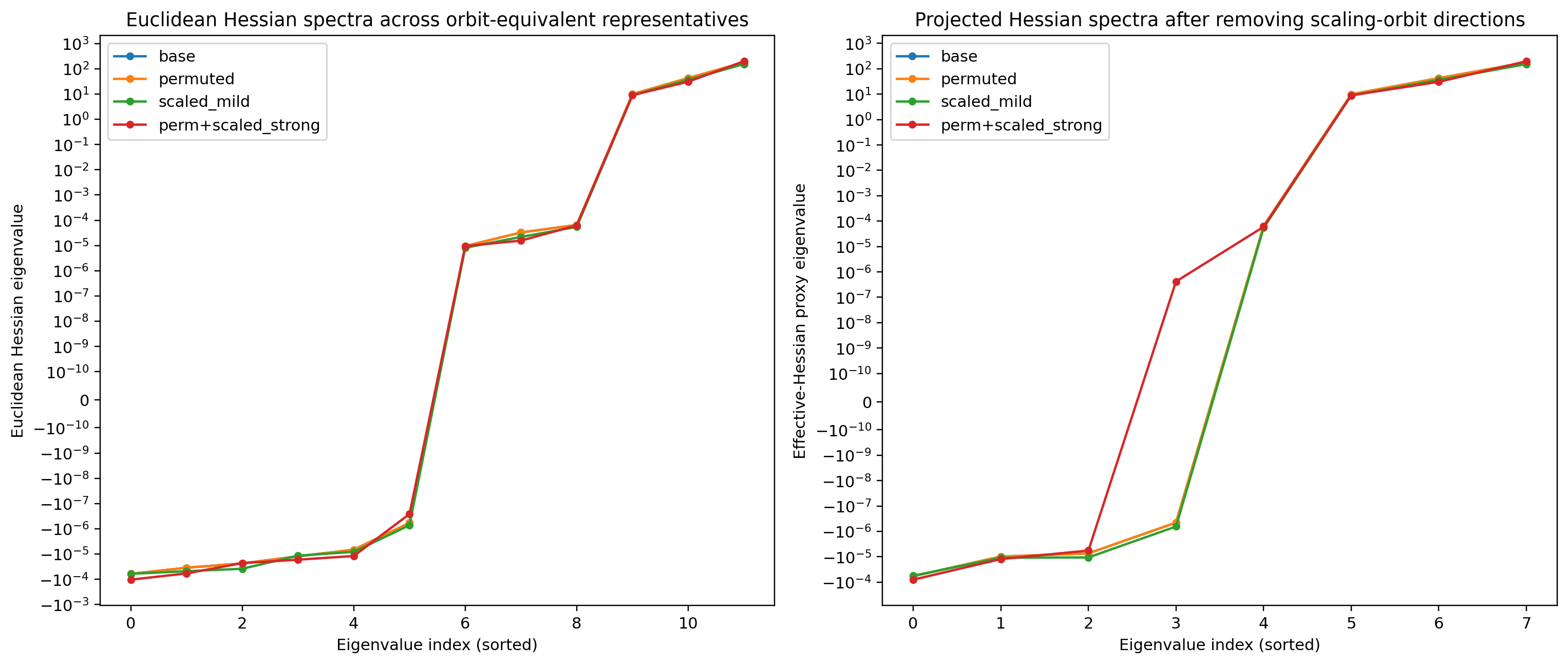}
    \caption{Euclidean Hessian spectra across orbit-equivalent representatives (left) and projected Hessian spectra after removing scaling-orbit directions (right). The variation in the Euclidean spectrum reflects representation dependence, while the projected spectrum is invariant.}
    \label{fig:false-flatness-projection}
\end{figure}

\begin{figure}[H]
    \centering
    \includegraphics[width=\linewidth]{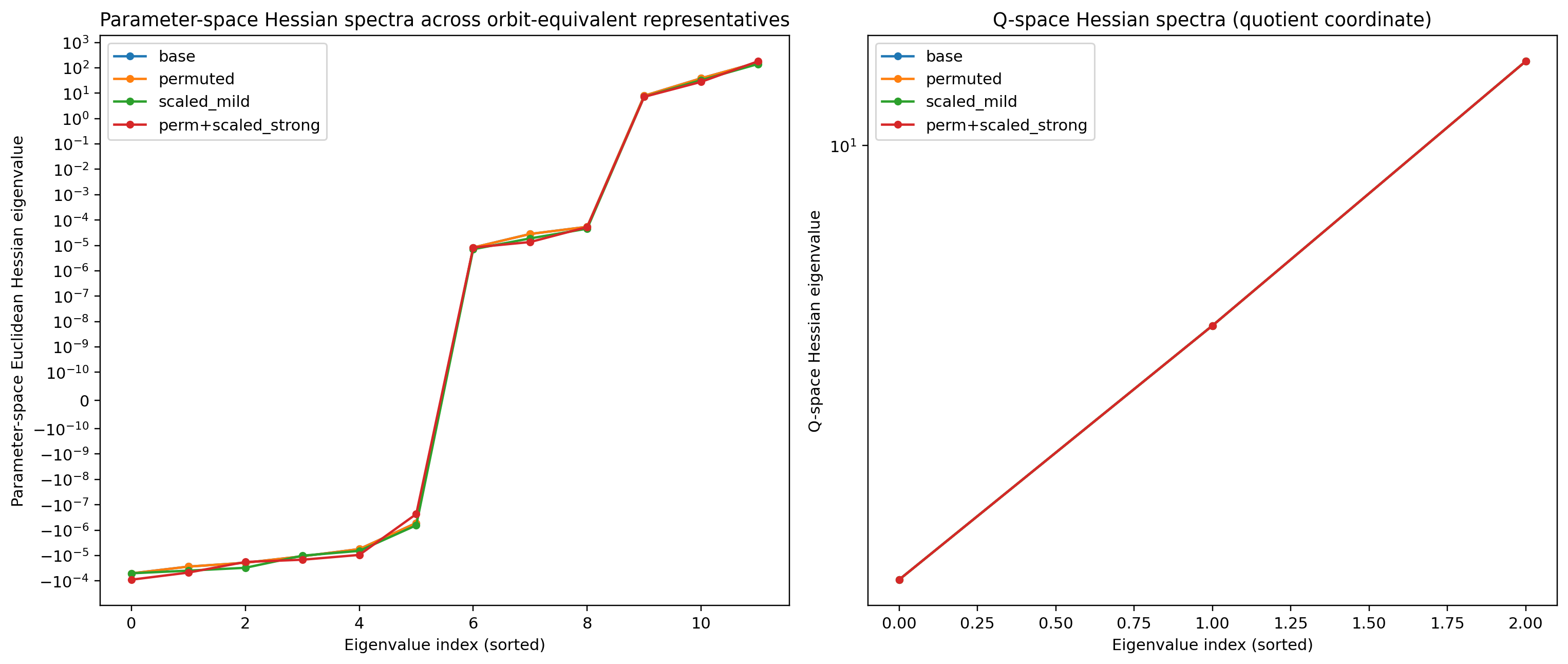}
    \caption{Parameter-space Hessian spectra vary across orbit-equivalent representatives (left), whereas $Q$-space Hessian spectra are numerically identical (right), confirming the intrinsic nature of quotient-level curvature.}
    \label{fig:hessian-invariance}
\end{figure}

\begin{figure}[H]
    \centering
    \includegraphics[width=\linewidth]{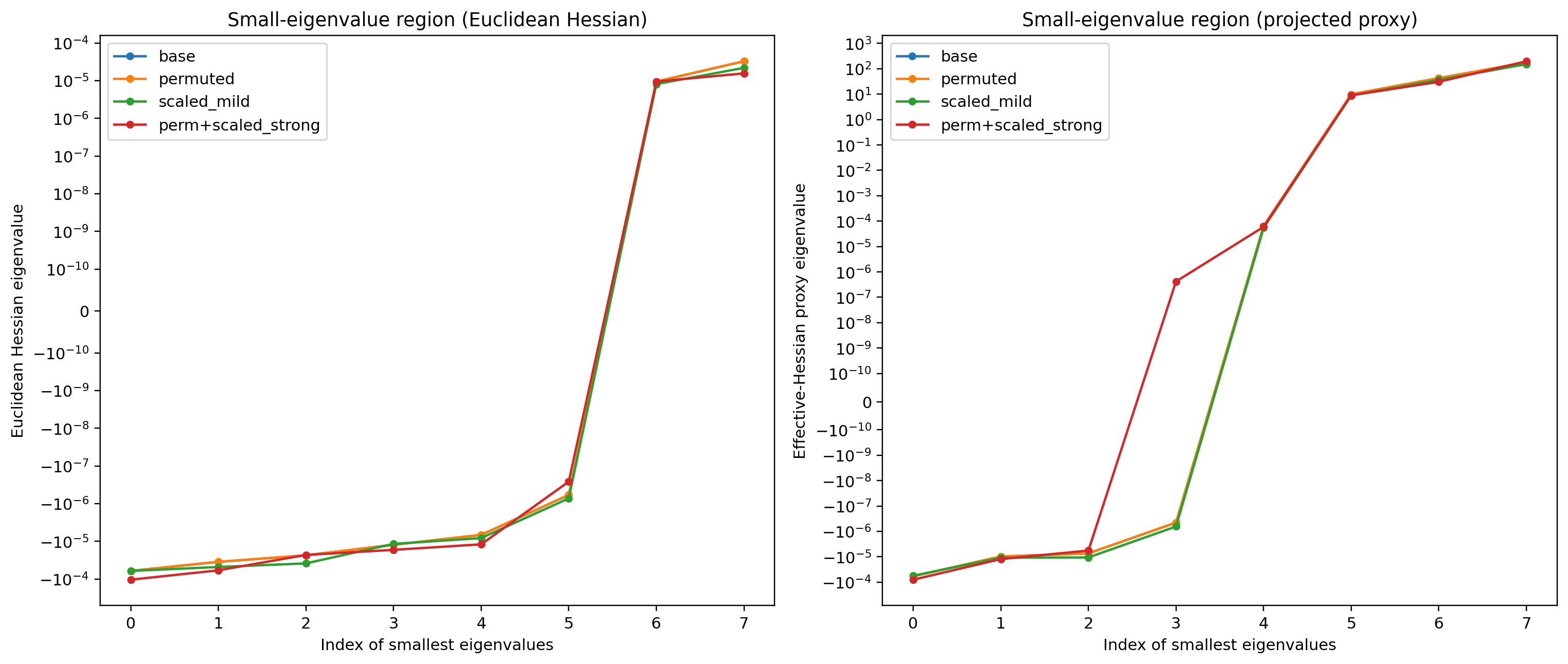}
    \caption{Zoomed view of the small-eigenvalue region: Euclidean Hessian (left) changes under rescaling, while the projected proxy (right) remains stable, demonstrating that apparent flatness stems from symmetry-induced redundancy.}
    \label{fig:small-eigenvalues-zoom}
\end{figure}

\subsection{Quotient-level curvature and local dynamics}

Our second experiment addresses the relation between local optimization behavior and curvature. The theoretical claim of Section \ref{sec:4} is not that every aspect of Euclidean training is determined by quotient geometry, but rather that the predictor-relevant local geometry should be more naturally organized by quotient-level curvature than by raw ambient curvature. To make this effect numerically visible, we work with quadratic classification under logistic loss rather than squared loss: with logistic loss, the Hessian in $Q$-coordinates genuinely varies with the point $Q$, so quotient-level curvature can meaningfully be compared across nearby states. We evaluate local behavior at an intermediate checkpoint before the loss enters the saturated near-zero regime, where second-order quantities become numerically degenerate and local decay is too weak to be informative. Around this checkpoint we construct two classes of nearby points: orbit-equivalent controls, obtained by permutation and mild rescaling and therefore preserving $Q$, and nearby non-equivalent perturbations, which change $Q$ while remaining in a comparable loss band. For each point we compute the Euclidean Hessian in parameter space, the Hessian in $Q$-space, several spectral summaries, and a short-run empirical log-loss decay rate under gradient descent.

The stable conclusion is that the most informative descriptors of short-run local decay are not ambient condition numbers, but quotient-level curvature magnitudes. Across multiple random seeds, the trace of the $Q$-space Hessian, its Frobenius norm, and its smallest positive eigenvalue show the strongest and most consistent relation to the empirical decay rate, whereas ambient parameter-space condition numbers are less stable and less informative. This is exactly the pattern visible in the pooled scatter plots: parameter-space condition numbers do not organize the local decay rates cleanly, while quotient-level summaries separate the observed regimes much more coherently. The correct interpretation is also conceptually natural. A condition number measures anisotropy, whereas the short-run descent speed in this local experiment is governed more directly by the overall magnitude of local curvature. The evidence therefore supports a refined version of the theoretical message of Section \ref{sec:4}: local dynamics are better organized by quotient-level curvature summaries than by ambient curvature descriptors tied to a redundant coordinate system.

The orbit-equivalent controls provide an additional internal consistency check. For such controls, the quotient coordinate $Q$ is unchanged, so quotient-level curvature is essentially unchanged as well, and the observed short-run dynamics remain nearly identical. The ambient parameter-space curvature can nevertheless vary slightly across these representatives. This is exactly what one expects if predictor evolution is governed by quotient-level geometry, while the ambient parameterization contributes only secondary numerical variation that should not be interpreted as a change in the intrinsic local landscape.

% --- Insert Figure for Section 6.3 ---
\begin{figure}[H]
    \centering
    \includegraphics[width=\linewidth]{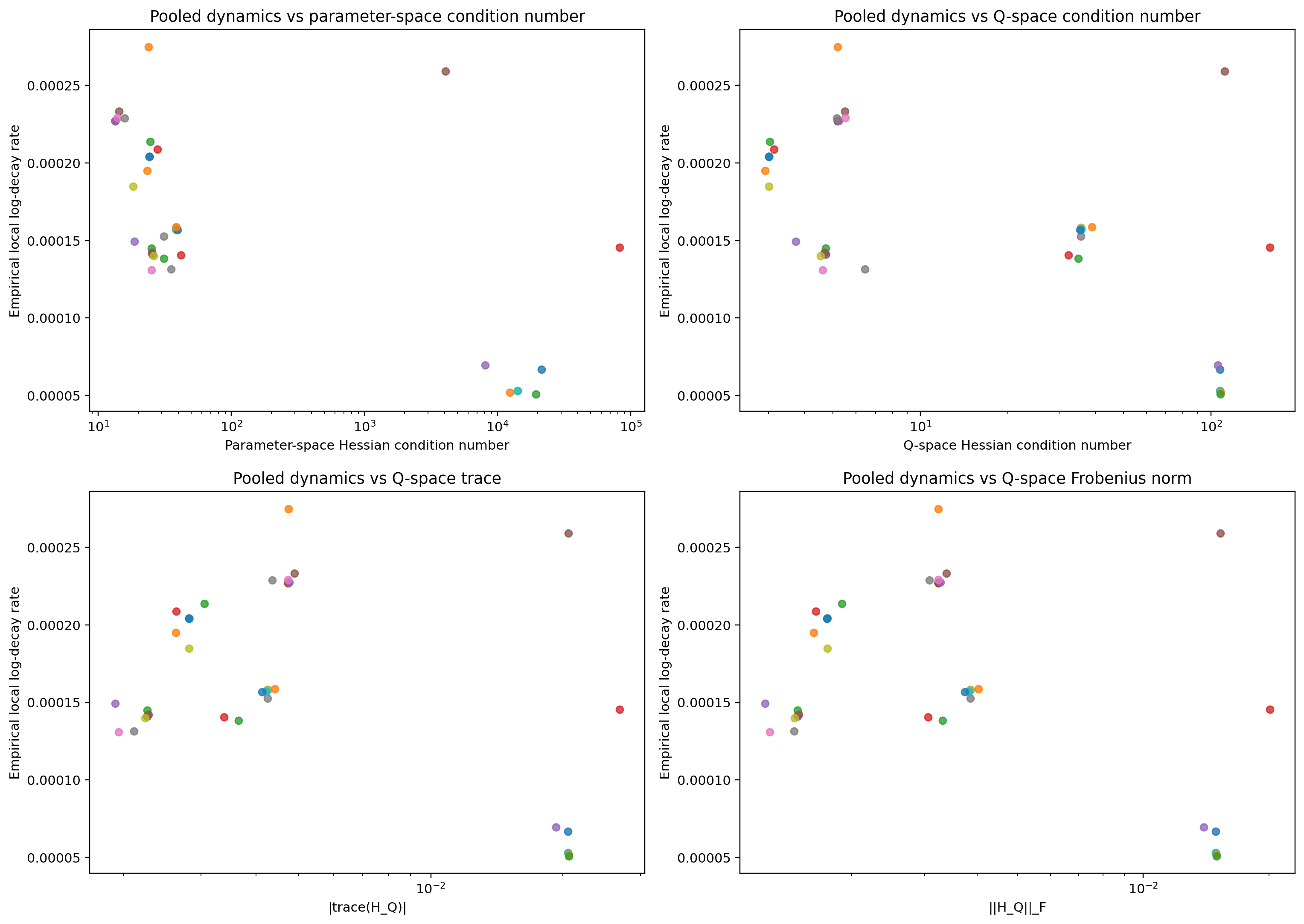}
    \caption{Pooled scatter plots of short-run empirical log-loss decay rates versus various curvature descriptors. Quotient-level quantities (e.g., $Q$-space trace and Frobenius norm) exhibit stronger correlation with local dynamics than ambient parameter-space condition numbers.}
    \label{fig:local-dynamics-curvature}
\end{figure}

\subsection{Quotient-level implicit bias in an underdetermined quadratic model}

Our third experiment addresses the implicit-bias perspective of Section \ref{sec:5}. To make quotient-level selection visible, it is necessary to work in a regime where the data do not determine a unique quotient solution. We therefore move to an underdetermined quadratic regression setting with
\[
n<\dim(\mathrm{Sym}(d))=\frac{d(d+1)}{2},
\]
so that multiple distinct matrices $Q$ can interpolate the training data. The experiment has two complementary parts. First, starting from a fixed trained solution, we construct several orbit-equivalent representatives using permutations and rescalings. Second, we repeat training from multiple random initializations in the underdetermined regime and compare the learned quotient objects across seeds.

The first part confirms the invariance claim that underlies the quotient-level formulation of complexity. Along a fixed orbit, gauge-dependent parameter complexities such as the Euclidean norm and path-like parameter complexity vary substantially across representatives, whereas quotient-level matrix quantities derived from $Q$, including $\|Q\|_F$ and $\|Q\|_*$, remain invariant up to numerical precision. The corresponding scatter plot makes this distinction explicit: ambient parameter complexity changes within the same orbit, while quotient complexity is stable at numerical precision. Thus complexity is not an intrinsic property of a particular representative in parameter space, but of the quotient object itself.

The second part reveals the genuinely quotient-level nature of implicit bias. In contrast to the fully determined setting of the earlier experiment, different random initializations now converge to different near-interpolating matrices $Q$. This is visible most clearly in the learned singular spectra, which vary substantially across seeds, with some runs displaying faster spectral decay, lower effective rank, or greater concentration in leading singular directions than others. At the same time, raw parameter norms do not provide a clean coordinate system for organizing these solutions: similar parameter norms can correspond to substantially different matrices $Q$, and path-like parameter complexities do not map transparently to the structure of the learned predictor. By contrast, matrix-level quantities such as the Frobenius norm, nuclear norm, stable-rank surrogate, and especially the singular spectrum directly describe how the learned quotient solutions differ. In this sense, the quotient-level selection problem becomes visible only after passing from $\theta$-space to $Q$-space.

This evidence supports the main qualitative claim of Section \ref{sec:5}. We do not claim that gradient descent in this setting exactly minimizes one specific convex matrix complexity such as the nuclear norm. The experiment is not designed to identify a universal closed-form variational principle. The point is more basic and more robust: once multiple quotient-feasible solutions exist, the meaningful comparison among learned solutions is naturally expressed in quotient coordinates. Implicit bias, insofar as it is visible here, acts at the level of predictor classes rather than gauge-dependent parameter representatives.

\begin{figure}[H]
    \centering
    \includegraphics[width=\linewidth]{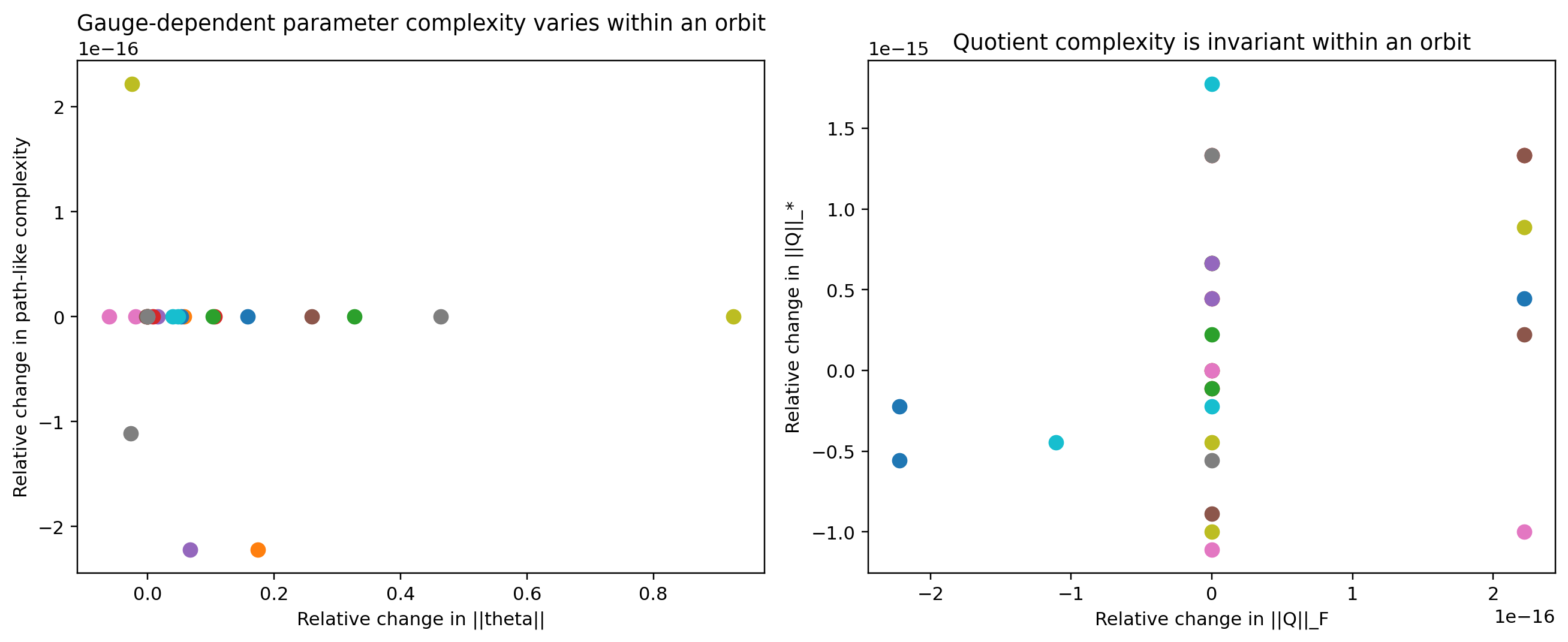}
    \caption{Gauge-dependent parameter complexity varies within a single symmetry orbit (left panel), while quotient-level complexity remains invariant (right panel), illustrating that meaningful complexity resides in the quotient object $Q$.}
    \label{fig:complexity-invariance-orbit}
\end{figure}

\begin{figure}[H]
    \centering
    \includegraphics[width=\linewidth]{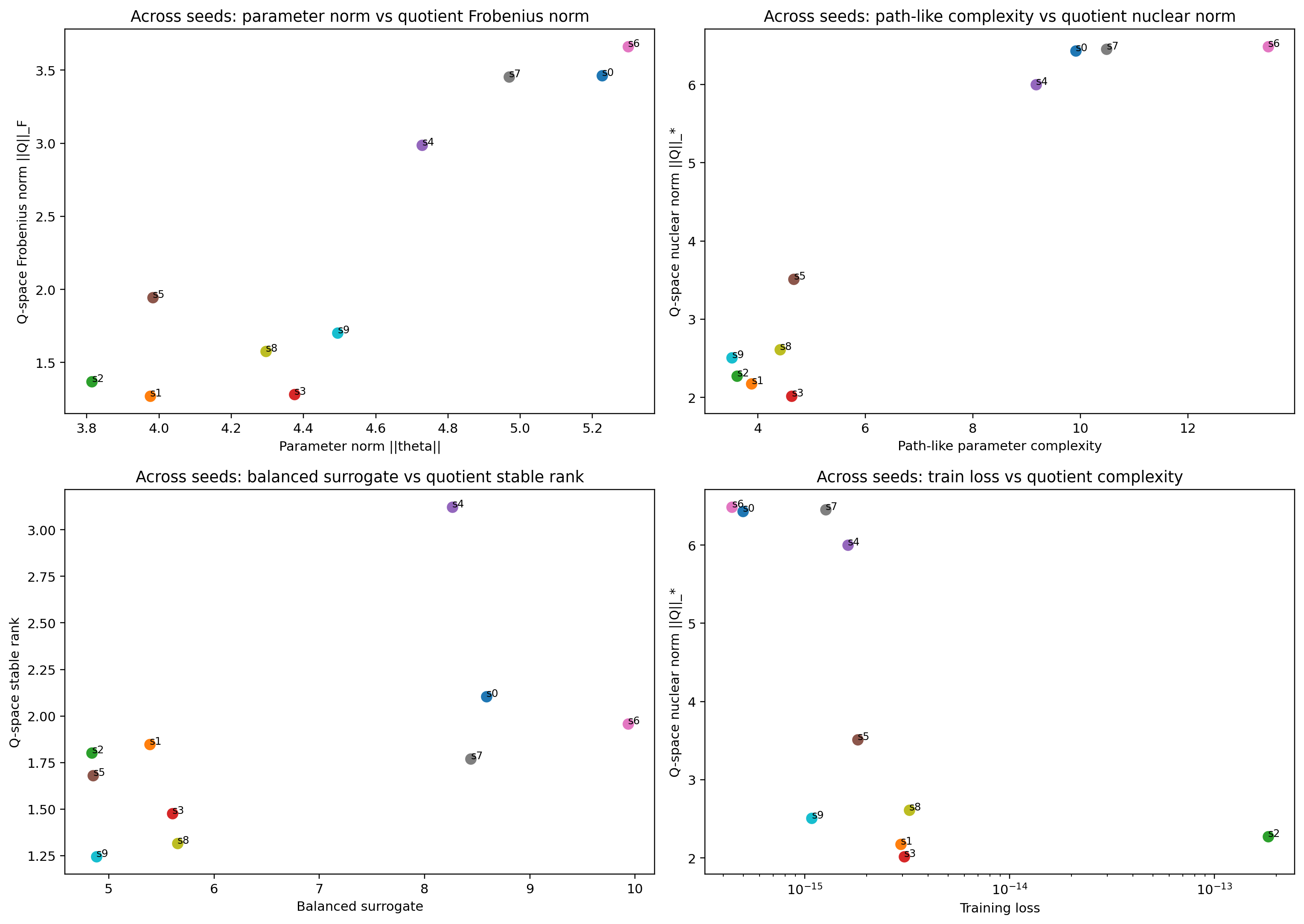}
    \caption{Across random seeds, raw parameter norms and path-like complexities show no coherent structure (left panels), whereas quotient-level quantities such as $\|Q\|_F$, $\|Q\|_*$, stable rank, and training loss reveal organized patterns, supporting a quotient-level view of implicit bias.}
    \label{fig:across-seeds-comparison}
\end{figure}

\begin{figure}[H]
    \centering
    \includegraphics[width=\linewidth]{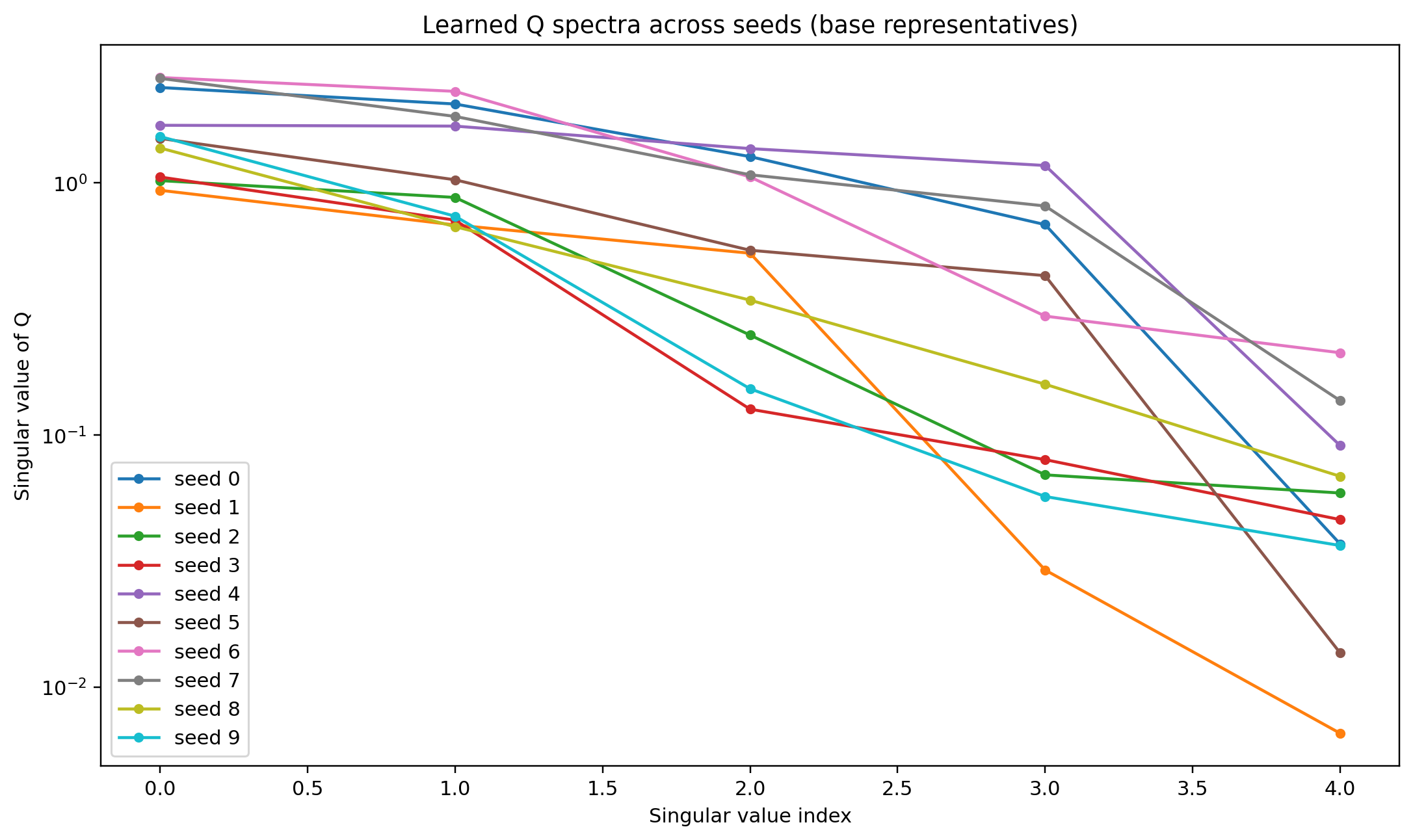}
    \caption{Learned singular spectra of $Q$ across different random initializations. Substantial variation in spectral decay and effective rank demonstrates that implicit bias operates at the level of the quotient object.}
    \label{fig:q-spectra-across-seeds}
\end{figure}

\subsection{Discussion of the numerical evidence}

Taken together, the experiments in this section support the geometric picture developed in the previous sections. First, ambient parameter-space curvature is representation-dependent, whereas quotient-coordinate curvature is intrinsic. Second, local optimization dynamics are better organized by quotient-level curvature magnitudes than by raw ambient condition numbers. Third, in underdetermined regimes the meaningful notion of complexity is attached to the quotient object $Q$, and implicit-bias questions are therefore naturally posed on the quotient manifold rather than in the redundant ambient parameterization. These three observations are exactly the experimentally robust consequences of the theory that the present quadratic model makes fully visible.

\section{Conclusion}
This paper develops a quotient-geometric framework for understanding symmetry, curvature, dynamics, and implicit bias in simple shallow neural networks. The main starting point is that overparameterized networks are not naturally described by their raw parameter space: hidden-unit permutations, rescalings, and related symmetries produce large equivalence classes of parameter vectors that realize the same predictor. Once this redundancy is taken seriously, several familiar phenomena in neural-network theory can be reinterpreted as geometric consequences of working in an excessively large ambient space.

Our first contribution is structural. On a regular set, we characterize the symmetry action and the resulting quotient space of predictor classes. This identifies the correct reduced state space on which locally identifiable predictors live. Our second contribution is geometric. Using the finite-sample realization map, we define a function-induced metric on the quotient and derive an effective curvature notion that removes orbit directions and isolates the intrinsic local geometry of the predictor class. This clarifies why Euclidean Hessians in parameter space can be highly degenerate even when the underlying predictor is not intrinsically flat. Our third contribution is dynamical. We show that gradient flow admits a vertical-horizontal decomposition, where vertical motion is pure gauge and horizontal motion governs function-level evolution. This provides a clean interpretation of optimization after symmetry reduction. Our fourth contribution is conceptual. We formulate implicit bias at the quotient level, arguing that meaningful complexity should be defined on predictor classes rather than on individual parameter representatives.

The quadratic-activation model provides a particularly transparent realization of this program. In that setting, the quotient object can be represented explicitly by a symmetric matrix ($Q$), making the symmetry-reduced geometry directly observable. The numerical experiments confirm the core theoretical predictions that are robustly visible in this model: ambient flatness is representation-dependent, quotient-coordinate curvature is intrinsic, local dynamics correlate more naturally with quotient-level curvature summaries than with raw ambient quantities, and in underdetermined regimes different feasible solutions are most meaningfully compared through matrix-level complexity rather than through gauge-dependent parameter norms.

Several directions remain open. First, it would be valuable to extend the quotient-metric and effective-curvature construction beyond the simplest shallow models to broader classes of homogeneous and partially homogeneous architectures. Second, the quotient-level implicit-bias principle developed here is qualitative; a sharper asymptotic characterization, analogous to max-margin or minimum-complexity results in other settings, would strengthen the theory considerably. Third, while our experiments make the quotient structure fully visible in quadratic models, developing numerically stable quotient-level curvature estimators for more general nonlinear shallow networks remains an important challenge. More broadly, our results suggest that many questions about optimization and generalization in overparameterized learning may be better posed after symmetry reduction. From this viewpoint, the quotient space of predictor classes is not merely a technical convenience, but the natural geometric arena for the theory of symmetric neural networks.

% Acknowledgements and Disclosure of Funding should go at the end, before appendices and references

\acks{This research received no external funding.}


\newpage

{\noindent \em Remainder omitted in this sample. See http://www.jmlr.org/papers/ for full paper.}

\vskip 0.2in
\bibliography{sample}

\end{document}